\documentclass[english]{article}

\usepackage[utf8]{inputenc}
\usepackage[T1]{fontenc}
\usepackage{babel}
\usepackage{caption}
\usepackage{subcaption}
\usepackage{graphicx}
\usepackage{amsmath}
\usepackage{amssymb,amsfonts}
\usepackage{amsthm}
\usepackage{dsfont}
\usepackage{mathtools}
\usepackage[ruled]{algorithm2e}
\usepackage{todonotes}
\usepackage{tikz}
\usepackage{xspace}
\usepackage{float}
\usepackage{wrapfig}
\usepackage{enumitem}
\usepackage{fullpage}

\usepackage[style=numeric]{biblatex}
\bibliography{biblio,ref,bibliography,vr2017}
\graphicspath{{}}
\usepackage{hyperref}
\hypersetup{
    colorlinks=true,
    citecolor=purple,
    linkcolor=red,
    filecolor=magenta,
    urlcolor=cyan,
    pdftitle={Overleaf Example},
    pdfpagemode=FullScreen,
    }
\usepackage{csquotes}

\newtheorem{thm}{Theorem}
\newtheorem{prop}{Proposition}
\newtheorem{lemma}[prop]{Lemma}
\newtheorem{remark}{Remark}

\theoremstyle{definition}

\newcommand{\R}{\mathbb{R}}
\newcommand{\E}{\mathbb{E}}
\newcommand{\V}{\mathbb{V}}
\newcommand{\proba}{\mathbb{P}}
\newcommand{\argmin}[1]{\underset{#1}{\textup{argmin}}\xspace}

\newcommand{\mdp}{\mathcal{M}}
\newcommand{\mdpb}{\mathcal{M'}}
\newcommand{\mdpt}{\widehat{\mathcal{M}}_t}
\newcommand{\KL}{\textup{KL}_{\mdp|\mdpb}}
\newcommand{\thetat}{\hat{\theta}_t}
\newcommand{\mut}{\hat{\mu}_t}
\newcommand{\Vt}{\widehat{V}_t}
\newcommand{\Qt}{\widehat{Q}_t}
\newcommand{\thetath}{\hat{\theta}_{t,h}}
\newcommand{\muth}{\hat{\mu}_{t,h}}
\newcommand{\VtH}[1]{\widehat{V}_{t,#1}}
\newcommand{\QtH}[1]{\widehat{Q}_{t,#1}}

\newcommand{\diffMDPb}[1]{\theta_\mdp-\theta_\mdpb + \gamma(\mu_\mdp-\mu_\mdpb)^\top{#1}}
\newcommand{\distMDPb}[1]{\left\|\diffMDPb{#1}\right\|}
\newcommand{\diffMDPt}[1]{\thetat-\theta_\mdp + \gamma(\mut-\mu_\mdp)^\top{#1}}
\newcommand{\distMDPt}[1]{\left\|\diffMDPt{#1}\right\|}

\newcommand{\diffMDPbH}[1]{\theta_{\mdp,h}-\theta_{\mdpb,h} + (\mu_{\mdp,h}-\mu_{\mdpb,h})^\top{#1}}
\newcommand{\distMDPbH}[1]{\left\|\diffMDPbH{#1}\right\|}
\newcommand{\diffMDPtH}[1]{\hat{\theta}_{t,h}-\theta_{\mdp,h} + (\hat{\mu}_{t,h}-\mu_{\mdp,h})^\top{#1}}
\newcommand{\distMDPtH}[1]{\left\|\diffMDPtH{#1}\right\|}

\newcommand{\cA}{\mathcal{A}}

\newcommand{\cC}{\mathcal{C}}

\newcommand{\cF}{\mathcal{F}}

\newcommand{\cL}{\mathcal{L}}
\newcommand{\cM}{\mathcal{M}}
\newcommand{\cN}{\mathcal{N}}
\newcommand{\cS}{\mathcal{S}}

\newcommand{\cV}{\mathcal{V}}
\newcommand{\cX}{\mathcal{X}}

\newcommand{\Alt}{\mathrm{Alt}}

\newcommand{\EE}{\mathbb{E}}
\newcommand{\PP}{\mathbb{P}}
\newcommand{\RR}{\mathbb{R}}

\usepackage{authblk}

\begin{document}

\title{Best Policy Identification in Linear MDPs}
%

\author[1,2]{Jérôme Taupin\thanks{jerome.taupin@ens.psl.eu}}
\author[1]{Yassir Jedra\thanks{jedra@kth.se}}
\author[1]{Alexandre Proutière\thanks{alepro@kth.se}}
\affil[1]{KTH Royal Institute of Technology}
\affil[2]{ENS}

\date{}

\maketitle

\begin{abstract}
We investigate the problem of best policy identification in discounted linear Markov Decision Processes in the fixed confidence setting under a generative model. We first derive an instance-specific lower bound on the expected number of samples required to identify an $\varepsilon$-optimal policy with probability $1-\delta$. The lower bound characterizes the optimal sampling rule as the solution of an intricate non-convex optimization program, but can be used as the starting point to devise simple and near-optimal sampling rules and algorithms. We devise such algorithms. One of these exhibits a sample complexity upper bounded by ${\cal O}({\frac{d}{(\varepsilon+\Delta)^2}} (\log(\frac{1}{\delta})+d))$ where $\Delta$ denotes the minimum reward gap of sub-optimal actions and $d$ is the dimension of the feature space. This upper bound holds in the moderate-confidence regime (i.e., for all $\delta$), and matches existing minimax and gap-dependent lower bounds. We extend our algorithm to episodic linear MDPs.            
\end{abstract}

\section{Introduction}

In Reinforcement Learning (RL), an agent interacts with an unknown controlled stochastic dynamical system, with the objective of identifying as quickly as possible an approximately optimal control policy. In this paper, we consider dynamical systems modelled through discounted or episodic Markov Decision Processes (MDPs), and investigate the problem of best policy identification in the {\it fixed confidence} setting. More precisely, we aim at devising $(\varepsilon,\delta)$-PAC RL algorithms, i.e., algorithms identifying $\varepsilon$-optimal policies with a level of certainty greater than $1-\delta$, using as few samples as possible. Such a learning objective has been considered extensively in tabular MDPs both in the discounted and episodic settings, most often using a minimax approach, see e.g. \cite{Kearns1999,kakade2003sample,even2006action,azar2013minimax,NIPS2018_7765,pmlr-v125-agarwal20b,li2020breaking,MinimaxOptimalDiscounted,SampleComplexityEpisodicFixedHorizon,EpisodicRLMinimaxRevisited} and more recently adopting an instance-specific analysis \cite{AdaptiveSamplingBestPolicyIdentification,NavigatingBestPolicy}. According to the aforementioned work, in tabular MDPs, the minimal sample complexity for identifying an $\varepsilon$-optimal policy with probability at least $1-\delta$ scales as ${SA\over \varepsilon^2}\log(1/\delta)$ (ignoring the dependence in the time-horizon or discount factor), where $S$ and $A$ represent the sizes of the state and action spaces respectively. These results illustrate the curse of dimensionality (tabular RL algorithms can address very small problems only), and highlight the need for the use of function approximation towards the design of scalable RL algorithms. 

\medskip
Despite the empirical successes of RL algorithms leveraging function approximation, and more specifically of deep RL algorithms, our theoretical understanding of these methods remain limited. In this paper, we investigate the {\it linear} MDPs, where linear functions are used to approximate the system dynamics and rewards. We propose computationally simple algorithms solving the best policy identification problem in the fixed confidence setting, and analyze their sample complexity. More precisely our contributions are as follows. 

\begin{enumerate}
\item For linear MDPs with discount factor $\gamma$, we first derive instance-specific sample complexity lower bounds satisfied by any $(\varepsilon,\delta)$-PAC algorithm. Inspired by these lower bounds, we develop GSS (G-Sampling-and-Stop), an $(\varepsilon,\delta)$-PAC algorithm that blends G-optimal design method and Least-Squares estimators. In the generative model (when in each round, the algorithm can sample a transition and reward from any (state, action) pair), we show that the expected sample complexity of GSS scales at most as $\frac{d(1-\gamma)^{-4}}{(\Delta(\mdp)+\varepsilon)^2} (\log(\frac{1}{\delta}) + d)$ (up to logarithmic factors), where $\Delta(\mdp)$ is an appropriately defined instance-specific sub-optimality gap that depends on the MPD $\mdp$. 

\item For linear MDPs with finite time horizon $H$, we apply the same approach as that used for discounted MDPs. We derive instance-specific sample complexity lower bounds, and based on these bounds, extend the design of GSS to the episodic setting. The analysis of GSS reveals that its expected sample complexity scales at most as $\frac{d H^4}{(\Delta(\mdp)+\varepsilon)^2} (\log(\frac{1}{\delta}) + d)$ (up to logarithmic factors).
\end{enumerate}

\medskip
The paper is organized as follows. We start with a literature review in the next section. Sections 3, 4, and 5 are devoted to discounted linear MDPs in the generative model.  Finally in Section 6, we extend our results to episodic linear MDPs.

\section{Related Work}

Linear models in RL have attracted a lot of attention over the last few years. We distinguish episodic and discounted MDPs. 

\medskip
{\bf Episodic linear MDPs.} Most of the studies have aimed at devising algorithms minimizing regret. Jin et al. \cite{jin2020provably} propose an optimistic Least Squares Value Iteration (LSVI) algorithm that achieves a regret upper bound of order $\tilde{\mathcal{O}}(\sqrt{d^3H^3 T})$ and that can be implemented in  polynomial time. He et al. \cite{he2021logarithmic} present UCRL-VTR, a confidence based algorithm adapted to the linear MDP setting. The algorithm achieves a gap dependent regret of order $\tilde{\mathcal{O}} ( \frac{d^2H^5}{\Delta_{\min}}\log\left(\frac{T}{\delta}\right)^3)$. 

When it comes to best policy identification problems, researchers have used different approaches. In \cite{wagenmaker2022reward}, Wagenmaker et al. aim at identifying an $\epsilon$-optimal policy identification objective. They establish a sample complexity minimax lower bound of order $\Omega( \frac{d^2 H^2}{\epsilon^2})$, and propose an a reward-free algorithms with sample complexity of order $\tilde{\mathcal{O}}(\frac{d(\log(1/\delta) + d)H^5}{\epsilon^2})$. 

In a subsequent work, Wagenmaker et al. \cite{wagenmaker2022instance} introduce PEDEL, an elimination based algorithm with instance-specific sample complexity guarantees. In the worst case, the sample complexity upper bound scales as $\tilde{\mathcal{O}}(\frac{d H^5 (d H^2 + \log(1/\delta))}{\epsilon^2})$. This bound hides a dependence on $\lambda_{\min}^\star$, the maximal minimum eigenvalue of the covariates matrix that can be induced by a policy. As in our work, the derived instance-specific sample complexity guarantees are related to G-optimal design and take the following form: 
$$
C_ 0 H^4 \sum_{h=1}^H \inf_{\Lambda_{exp}} \max_{\pi \in \Pi} \frac{ \Vert \phi_{\pi,h} \Vert_{\Lambda_{exp}^{-1}}}{\max \{ V^\star(\Pi) - V^\pi, \Delta_{\min}(\Pi), \epsilon^2  \}} \log\left(\frac{\vert \Pi\vert}{\delta}\right) +  C_1, 
$$
where $C_0 =   \log\left(\frac{1}{\epsilon}\right) \mathrm{polylog}\left(H, \log(1/\epsilon) \right)$ and $C_1= \mathrm{poly}\left(d, H, \frac{1}{\lambda_{\min}^\star}, \log(1/\delta), \log(1/\varepsilon), \log(\vert \Pi\vert )  \right)$. Note that PEDEL requires as input a set of policies $\Pi$. The authors propose a way to approximate the set of all policies using restricted linear soft-max policies $\Pi_\epsilon$ which leads to an overall sample complexity of order 
$$
C_0 H^4 \sum_{h=1}^H \inf_{\Lambda_{exp}} \max_{\pi \in \Pi_\epsilon} \frac{ \Vert \phi_{\pi,h} \Vert_{\Lambda_{exp}^{-1}}}{\max \{ V^\star - V^\pi, \epsilon^2  \}} (dH^2 + \log(\frac{1}{\delta})) +  C_1.
$$

In Zanette et al. \cite{zanette2019limiting}, the authors also investigate the problem of identifying an $\epsilon$-optimal policy with a generative model and propose a Linear Approximate Value Iteration algorithm (LAVI). They leverage the idea of anchor (state, action) pairs but require a set of such anchor pairs for each layer $h \in [H]$. 

\medskip
{\bf Discounted linear MDPs.} In \cite{yang2019sample}, Yang et al. focus on the $\epsilon$-optimal policy identification problem in the generative setting and present Phased Parametric Q-Learning (PPQ-learning), an algorithm with sample complexity of order $\tilde{\mathcal{O}}(\frac{d}{(1-\gamma)^3 \epsilon^2 } \log(\frac{1}{\delta}))$ under the restrictive assumption that a so-called set of (state, action) anchor pairs exist (see Assumption 2) and that it is of size $d$. More precisely, this assumption states that there exists $\mathcal{K} \subset \mathcal{S} \times \mathcal{A}$, a set of anchor (state, action) pairs such that for all $(s,a) \in \mathcal{S} \times \mathcal{A}$, $\phi(s,a)$ can be written as convex combination of features of anchor pairs. The authors further assume that $\vert \mathcal{K}\vert = d$ and that all features have non-negative entries and that the features correspond to probability vectors. The authors finally provide a matching minimax lower bound of order $\tilde{\Omega}( \frac{d}{\epsilon^2 (1-\gamma)^3})$.   

Lattimore et al. \cite{lattimore2020learning} also consider the $\epsilon$-optimal policy identification problem in the generative setting. They devise a sampling rule based on G-optimal design and use an approximate policy iteration algorithm to recover the optimal policy. Their algorithm seeks to estimate the Q function directly at each iteration, by first evaluating the value of Q at anchor (state, action) pairs (determined by the G-optimal design) via rollout, and by then generalizing using least squares. The sample complexity of their algorithm is of the order $\tilde{\mathcal{O}}( \frac{d\sqrt{d}}{\epsilon^2 (1-\gamma)^8} \log(\frac{1}{\delta}))$. 

Finally it is worth mentioning \cite{zhou2021provably}, where Zhou et al. consider the regret minimization problem in the forward model. The notion of regret for discounted MDPs is not easy to define. Here, the authors consider the accumulated difference of rewards between an Oracle policy and the proposed policy but along the trajectory followed under the latter policy (this policy could well lead the system into regions of the state space). The proposed algorithm achieves a regret scaling at most as $\hat{\mathcal{O}}(d\sqrt{T}/(1-\gamma)^2)$.

\section{Models and Objectives}

\subsection{Discounted linear MDPs}

We consider an infinite time-horizon MDP with a set of states $\cS$ and a set of actions $\cA$. Each (state, action) pair $(s,a)$ is associated to a feature $\phi(s,a)\in \mathbb{R}^d$, and we assume that the feature map $\phi$ is known to the learner. Without loss of generality, we assume that $\| \phi(s,a)\| \le 1$ for all $(s,a)\in {\cal S}\times {\cal A}$, and that the features $(\phi(s,a))_{s\in {\cal S},a\in {\cal A}}$ cover $\mathbb{R}^d$. An MDP ${\cal M}$ is defined through its dynamics $p_{\cal M}$ and reward distributions $q_{\cal M}$. More precisely, starting from state $s$ and given that action $a$ is selected, the probability to move to state $s'$ is $p_{\cal M}(s,a,s')$ and the distribution of the collected reward is $q_\mdp(s,a,\cdot)$. We assume that $q_\mdp(s,a,\cdot)$ is absolutely continuous w.r.t. a measure $\nu$ with support included in $[0,1]$, and we denote by $r_\mdp(s,a) = \E[q_\mdp(s,a)]$ the expected reward of the (state, action) pair $(s,a)$. The MDP is linear, which means that its dynamics and expected rewards can be parametrized as follows:
\begin{align}\label{eq:linmdp}
\forall (s,s')\in {\cal S}, \ \forall a\in {\cal A}, \quad p_\mdp(s,a,s') = \phi(s,a)^\top \mu_\mdp(s') \quad\hbox{and}\quad r_\mdp(s,a) = \phi(s,a)^\top \theta_\mdp,
\end{align}
where $\mu_\mdp$ is a family of $d$ measures over $\cS$, seen as a $\cS\times d$ dimensional matrix, and $\theta_\mdp\in\R^d$. We will assume that $\|\theta_\mdp\| \le \sqrt{d}$ and $\big\|\sum_{s\in\cS}|\mu_\mdp(s)|\big\| \le \sqrt{d}$.
We denote by $\mathbb{M}$ the set of linear MDPs, i.e., satisfying (\ref{eq:linmdp}). 

\medskip
A control policy $\pi$ maps states to actions. We denote by $s_t^\pi$ the state at time $t$ under the policy $\pi$. For a given discount factor $\gamma\in (0,1)$, the performance of a policy $\pi$ is expressed through its state value function $V_\mdp^\pi$ and its (state, action) value function $Q_\mdp^\pi$ defined by: for all $(s,a)\in {\cal S}\times {\cal A}$,
\begin{align*}
V_\mdp^\pi(s) = \E_\mdp\left[ \sum_{t=0}^{+\infty} \gamma^tr_\mdp(s_t^\pi,\pi(s_t^\pi)) | s_0^\pi=s \right] \quad\hbox{and}\quad
Q_\mdp^\pi(s,a) = r_\mdp(s,a) + \gamma \sum_{s'}p_\mdp(s,a,s')V_\mdp^\pi(s').
\end{align*}

A policy $\pi$ is said optimal for the MDP $\mdp$ if it maximizes the value function for any state, i.e., for any policy $\pi'$, we have $V_\mdp^\pi \ge V_\mdp^{\pi'}$ point-wise. Throughout the paper, we assume that the optimal policy $\pi^\star_\mdp$ is unique. The state and (state, action) value functions of $\pi^\star_\mdp$ are referred to as the value function $V_\mdp^\star$ and the Q function $Q_\mdp^\star$, respectively. A policy $\pi$ is said $\varepsilon$-optimal if $V_\mdp^\pi \ge V_\mdp^\star - \varepsilon$ point-wise, and we denote by $\Pi_\varepsilon^\star(\mdp)$ the set of $\varepsilon$-optimal policies of $\mdp$. Note that all our results apply to optimal policies by choosing $\varepsilon=0$.

\subsection{Best policy identification}

We aim at designing a learning algorithm interacting with the MDP ${\cal M}$ so as to identify an $\varepsilon$-optimal policy as quickly as possible. We formalize this design in a PAC framework. A learning algorithm consists of (i) a sampling rule, (ii) a stopping rule and (iii) a decision rule. 

\medskip
{\bf Sampling rule.} We distinguish between the {\it generative} and the {\it forward} model. In the former, in each round $t$, the sampling rule may select any (state, action) $(s_t,a_t)$ to explore depending on past observations. In the latter, the learner is forced to follow the trajectory of the system, and only the action may be selected. From the selected pair, the learner observes the next state and receives a sample of the corresponding reward. 

\medskip
{\bf Stopping rule.} This rule is defined through a stopping time $\tau$ deciding when the learner stops gathering information and wishes to output an estimated $\varepsilon$-optimal policy.

\medskip
{\bf Decision rule.} Based on the observations gathered before stopping, the learner outputs an estimated optimal policy $\hat{\pi}$.

\medskip
An algorithm is $(\delta,\varepsilon)$-PAC if it outputs a $\varepsilon$-optimal policy with probability at least $1-\delta$. Our goal is to design such algorithm with minimal expected sample complexity $\mathbb{E}_{\cal M}[\tau]$. In contrast with most existing analyses, we will derive {\it instance-specific} lower and upper bounds on the sample complexity of $(\delta,\varepsilon)$-PAC algorithms. In particular, we wish these bounds to depend on the sub-optimality gap of the MDP ${\cal M}$ defined by $\Delta(\mdp) = \min_{s\in\cS, a\neq\pi_\mdp^\star(s)} (V_\mdp^\star(s) - Q_\mdp^\star(s,a))$.

\section{Sample Complexity Lower Bounds}

To state our instance-specific lower bounds, we first introduce the following notations. Given two MDPs $\mdp$ and $\mdpb$ in $\mathbb{M}$, we write $\mdp\ll\mdpb$ if for every pair $(s,a)\in\cS\times\cA$, we have $p_\mdp(s,a,\cdot)\ll p_\mdp(s,a,\cdot)$ and $q_\mdp(s,a,\cdot)\ll q_\mdpb(s,a,\cdot)$. In this case, we define the Kullback-Leibler divergence between $\mdp$ and $\mdpb$ by:
\begin{equation}
    \label{eq:KLMDP}
    \KL(s,a) = \text{KL}(q_\mdp(s,a,\cdot)\|q_\mdpb(s,a,\cdot)) + \text{KL}(p_\mdp(s,a,\cdot)\|p_\mdpb(s,a,\cdot)).
\end{equation}

We also denote $\textup{kl}(a,b)$ the Kullback-Leibler divergence of two Bernoulli distributions of respective means $a$ and $b$. Finally, we introduce the following set of MDPs. This set include MDPs for which efficient policies for $\mdp$ are not $\varepsilon$-optimal.
\begin{align}
\Alt_\varepsilon(\mdp) & = \{\mdpb\in \mathbb{M} : \mdp\ll\mdpb, \Pi_\varepsilon^\star(\mdp) \cap \Pi_\varepsilon^\star(\mdpb) = \emptyset\}.
\end{align}
We refer to $\Alt_\varepsilon(\mdp)$ as the {\it set of alternative MDPs w.r.t. $\mdp$}.
The next proposition is established using classical change-of-measure arguments, and specifically, considering that the observations are generated under an MDP in the set of alternative MDPs w.r.t. $\mdp$. 

\begin{prop}
\label{prop:LowerBoundDiscounted}
Let $(\varepsilon, \delta)\in (0,1)^2$. The sample complexity $\tau$ of any $(\delta,\varepsilon)$-PAC algorithm satisfies $\E_\mdp[\tau] \ge T^\star(\mdp) \textup{kl}(\delta,1-\delta)$, where $T^\star(\mdp)^{-1}=\sup_{\omega\in\Sigma_{\cS\times\cA} }T(\mdp,\omega)^{-1}$ and 
\begin{align}
    T(\mdp,\omega)^{-1} = \inf_{\mdpb\in\Alt_\varepsilon(\mdp)} \sum_{s,a} \omega_{s,a}\KL(s,a).
\end{align}
\end{prop}
The vector $\omega\in \Sigma_{\cS\times\cA}$ solving the optimization problem leading to $T^\star(\mdp)$ can be interpreted as the optimal proportions of times an optimal algorithm should sample the various (state, action) pairs. As it turns out, as in the case of tabular MDPs \cite{AdaptiveSamplingBestPolicyIdentification}, analyzing and computing this allocation is difficult. Instead, our strategy will be to derive (tight) instance-specific and tractable upper bounds of the lower bound, and devise algorithms based on these upper bounds and the corresponding allocations. To state the upper bounds, we introduce the following quantities: let $\omega\in \Sigma_{\cS\times\cA}$,
\begin{align}
\Lambda(\omega) &= \sum_{(s,a)\in\cS\times\cA} \omega_{s,a} \phi(s,a)\phi(s,a)^\top,\\
\sigma(\omega) & = \max_{(s,a)\in\cS\times\cA} \|\phi(s,a)\|_{\Lambda(\omega)^{-1}}^2.
\end{align}
$\Lambda(\omega)$ is referred to as the {\it feature matrix}. Further introduce $\textup{Var}_{s,a}(\mdp)$ the variance of the r.v. $R+\gamma V^\star_{\cal M}(s')$ where $R$ is the random reward collected with (state, action) pair $(s,a)$ and $s'$ is a state with distribution $p_{\mdp}(s,a,\cdot)$. This variance is bounded by $(1-\gamma)^{-2}/4$ in the worst case. Finally, define $\textup{Var}(\mdp) = \max_{s,a}\textup{Var}_{s,a}(\mdp)$. The next theorem presents two upper bounds on the sample complexity lower bound. The first bound depends on the sub-optimality gap only, whereas the refined second upper bound also exhibits a dependence in the variance. 

\begin{thm}
\label{thm:UpperBoundDiscounted}
We have for all $\omega\in \Sigma_{\cS\times\cA}$,
\begin{align}
T(\mdp,\omega) & \le \frac{10\sigma(\omega)}{3(1-\gamma)^4(\Delta(\mdp)+\varepsilon)^2} := U(\mdp,\omega),\\
T^\star(\mdp) & \le \frac{10d}{3(1-\gamma)^4(\Delta(\mdp)+\varepsilon)^2} := U^\star(\mdp).
\end{align}
\end{thm}

Note that the optimal allocation $\omega^\star$ such that $U^\star(\mdp) =  \inf_{\omega\in \Sigma_{\cS\times\cA}} U(\mdp,\omega) = U(\mdp,\omega^\star)$ is characterized by $\sigma(\omega^\star)=\inf_{\omega\in \Sigma_{\cS\times\cA}}\sigma(\omega)$ and hence depends on the MDP ${\cal M}$ through its features only. The allocation $\omega^\star$ can be easily computed even before the learning process starts, and as shown in the next section, an algorithm just tracking it yields an expected sample complexity roughly no greater than $U^\star(\mdp)\textup{kl}(\delta,1-\delta)$.

\section{The GSS Algorithm}

This section presents the GSS (G-Sampling-and-Stop) algorithm. It samples (state, action) pair according to the allocation $\omega^\star$ minimizing $\sigma(\omega)$, and hence can be seen as a classical G-sampling strategy. The design of the stopping rule is driven by the upper bound of the sample complexity lower bound $U^\star(\mdp)\textup{kl}(\delta,1-\delta)$. However since $\mdp$ is unknown, we will replace, in this upper bound, $\mdp$ by the MDP obtained by plugging the least-squares estimators of $\theta_\mdp$ and $\mu_\mdp$. The components of GSS and their analysis is detailed below. 

\subsection{Sampling rule}

Under the GSS algorithm, we start by computing the allocation $\omega^\star$ minimizing $\sigma(\omega)$ over $\omega$. Then in each round $t$, the algorithm samples the pair $(s_t,a_t)$ according to $\omega^\star$. We define $\Phi_t = \begin{pmatrix}\phi(s_1,a_1) & \cdots & \phi(s_t,a_t) \end{pmatrix}^\top \in \R^{t\times d}$ the matrix of the first $t$ sampled features and $\omega_t = N_t/t$ the frequency of sampling of each pairs up to time $t$ ($N_t$ is the $SA$-dimensional vector recording the numbers of times each (state, action) pair has been selected up to round $t$). Notice that
\begin{equation}
    \frac{1}{t}\Phi_t^\top\Phi_t = \frac{1}{t}\sum_{\ell=1}^t \phi(s_\ell,a_\ell)\phi(s_\ell,a_\ell)^\top = \sum_{s,a} \frac{N_t(s,a)}{t} \phi(s,a)\phi(s,a)^\top = \Lambda(\omega_t).
\end{equation}

Following this sampling rule, the above matrix will converge towards $\Lambda(\omega^\star)$ and in particular $\sigma(\omega_t)$ will converge towards $\sigma(\omega^\star)$. Specifically, we have:

\begin{prop}
\label{prop:ConcentrationLambda}
Let $\delta\in(0,1)$. For any $t\ge 10d\log\big(\frac{2d}{\delta}\big)$,
\begin{equation}
    \proba\big[\sigma(\omega_t) \le 2\sigma(\omega^\star)\big]
    \ge 1-\delta.
\end{equation}
\end{prop}

\subsection{Least-squares estimators}

GSS leverages the least-squares estimators of the parameters $\mu_\mdp$ and $\theta_\mdp$. We provide below explicit expressions for these estimators and derive concentration inequalities characterizing their performance. When the algorithm selects (state, action) pair $(s_t,a_t)$ in round $t$, it observes the next state $s_t'$ and receives the reward $r_t$. Overall, in round $t$, the algorithm gathers the experience $(s_t,a_t,r_t,s_t')$. Define $R_t = \begin{pmatrix} r_1 & \cdots & r_t \end{pmatrix}^\top$ and $S_t(s) = (1_{s=s'_1}, \cdots ,1_{s=s'_t})^\top$. The regularized least-squares estimator with parameter $\lambda$  of $\theta_\cM$ after $t$ experiences is given by  
\begin{equation}
    \label{eq:ThetaLSE}
    \thetat
    = \argmin{\theta\in\R^d} \sum_{\ell=1}^t \left(r_\ell - \phi(s_\ell,a_\ell)^\top\theta\right)^2 + \lambda \|\theta\|^2\\
    = \left(\Phi_t^\top\Phi_t + \lambda I_d\right)^{-1} \Phi_t^\top R_t
\end{equation}
and the regularized least-squares estimator of $\mu_\mdp(s)$ with parameter $\lambda$ by
\begin{equation}
    \label{eq:MuLSE}
    \mut(s)
    = \argmin{\mu\in\R^d} \sum_{\ell=1}^t \left(1_{\lbrace s=s'_\ell \rbrace } - \phi(s_\ell,a_\ell)^\top\mu\right)^2 + \lambda \|\mu\|^2\\
    = \left(\Phi_t^\top\Phi_t + \lambda I_d\right)^{-1} \Phi_t^\top S_t(s).
\end{equation}
We will choose $\lambda = 1/d$ and denote $\mdpt$ the associated MDP and $\Vt$ and $\Qt$ its value functions.
The least square estimators can be controlled in the following sense :

\begin{prop}
\label{prop:ConcentrationLSEDiscounted}
Let $\delta\in(0,1)$. Regardless of the sampling rule, we have with probability at least $1-\delta$ that for all $t\ge1$,
\begin{equation}
    \distMDPt{\Vt^\star}_{t\Lambda(\omega_t)}^2 \le \frac{2}{(1-\gamma)^2} \left(2\log\left(\frac{\sqrt{e}\zeta(2)t^2}{\delta}\right) + d\log\left(8e^4dt^2\right)\right).
\end{equation}
\end{prop}

\subsection{Stopping and decision rules}

Denote $Z(t) = t\,U\big(\mdpt,\omega_t\big)^{-1}$ the quantity we seek to control in order to achieve the desired sample complexity.
The stopping rule of GSS is defined through the threshold
\begin{equation}
    \beta(\delta,t) = \frac{12}{5} \left(2\log\left(\frac{\sqrt{e}\zeta(2)t^2}{\delta}\right) + d\log\left(8e^4dt^2\right)\right)
\end{equation}
and the stopping time
\begin{equation}
    \tau = \inf\left\{t \ge 1 : Z(t) > \beta(\delta,t)\right\}.
\end{equation}

This stopping rule is inspired by classical log-likelihood based stopping rules. Usually such stopping rules are proved to be correct by controlling the Kullback-Leibler divergence between the empirical estimator of the MDP and the MDP itself.
Here we will take advantage of the linear assumptions to reduce the correctness of the stopping rule to proposition \ref{prop:ConcentrationLSEDiscounted}, which eventually leads to dependencies in the size of the feature space $d$ instead of $|\cS||\cA|$.
We use the least-squares estimators of the MDP parameters to compute $\mdpt$ and implement the stopping time. When the algorithm stops, it computes $\hat{\pi}$ the optimal policy for the MDP $\widehat{\mdp}_\tau$. The description of GSS is now complete and summarized in Algorithm \ref{alg:GSS}. 

\medskip
\begin{algorithm*}[H]\label{alg:GSS}
\SetAlgoLined
\DontPrintSemicolon
Compute $\omega^\star = \argmin{\omega\in\Sigma_{\cS\times\cA}}\sigma(\omega)$\;
\While{$Z(t) \le \beta(\delta,t)$}{
    Sample $(s_t,a_t)$ according to $\omega^\star$ and observe the corresponding experience \;
    Update $\thetat$ and $\mut$ according to \eqref{eq:ThetaLSE} and \eqref{eq:MuLSE}\;
    $t = t+1$\;
    }
    \Return $\hat{\pi} = \pi_t^\star$ the optimal policy of $\mdpt$\;
\caption{The GSS algorithm}
\end{algorithm*}

\subsection{Performance analysis}

The next theorem states that GSS is $(\varepsilon, \delta)$-PAC as long as it stops and is a direct consequence of proposition \ref{prop:ConcentrationLSEDiscounted}.

\begin{thm}
\label{thm:StoppingCorrectnessDiscounted}
Under the GSS algorithm, we have:
$\proba\left[\tau<+\infty, \hat{\pi} \notin \Pi_\varepsilon^\star(\mdp)\right] \le \delta$.
\end{thm}

\begin{proof}[Proof of Theorem \ref{thm:StoppingCorrectnessDiscounted}]
The proof of theorem \ref{thm:UpperBoundDiscounted} presents an intermediate bound
\begin{align*}
    U\big(\mdpt,\omega_t\big)^{-1}
    \le \inf_{\mdpb : \pi_t^\star \notin \Pi_\varepsilon^\star(\mdpb)}\frac{6(1-\gamma)^2}{5} \left\|\thetat-\theta_\mdpb + \gamma(\mut-\mu_\mdpb)^\top \Vt^\star\right\|_{\Lambda(\omega_t)}^2
    \le T\big(\mdpt,\omega_t\big)^{-1}.
\end{align*}
Under the event that $\pi_t^\star \notin \Pi_\varepsilon^\star(\mdp)$, we can then write
\begin{align*}
    Z(t) = t\,U\big(\mdpt,\omega_t\big)^{-1}
    \le \frac{6(1-\gamma)^2}{5} \distMDPt{\Vt^\star}_{t\Lambda(\omega_t)}^2.
\end{align*}
It follows that
\begin{align*}
    \PP\left[\tau<+\infty, \hat{\pi} \notin \Pi_\varepsilon^\star(\mdp)\right]
    &= \PP\left[\exists t\ge1 : Z(t) > \beta(\delta,t), \pi_t^\star \notin \Pi_\varepsilon^\star(\mdp)\right]\\
    &\le \PP\left[\exists t\ge1 : \distMDPt{\Vt^\star}_{t\Lambda(\omega_t)}^2 > \frac{5}{6(1-\gamma)^2}\beta(\delta,t)\right].
\end{align*}
The fact that this last probability is bounded by $\delta$ is exactly the statement of proposition \ref{prop:ConcentrationLSEDiscounted}.
\end{proof}

Finally, we analyze the expected sample complexity of GSS. In the following theorem, we are actually able to derive sample complexity upper bounds for any level of confidence $\delta$ (which contrasts with most existing analyses even in the case of best-arm identification problems in bandits). Asymptotically when $\delta$ goes to 0, the sample complexity upper bound scales as $U^\star(\mdp)\log(1/\delta)$. 

\begin{thm}
\label{thm:SampleComplexityDiscounted}
There exists universal constants $c_1$, $c_2$, $c_3$ such that under GSS,
\begin{align}
    \E[\tau]
    &\le c_1\frac{d(1-\gamma)^{-4}}{(\Delta(\mdp)+\varepsilon)^2} \left(\log\left(\frac{c_2}{\delta}\right) + d \log\left(c_3\frac{d^2(1-\gamma)^{-4}}{(\Delta(\mdp)+\varepsilon)^2}\right)\right).
\end{align}
\end{thm}
In particular this result implies that the algorithm stops almost surely. Together with theorem \ref{thm:StoppingCorrectnessDiscounted} this proves that the GSS algorithm is $(\delta,\varepsilon)$-PAC.

\section{Episodic Linear MDPs}

In this section, we extend our results to episodic MDPs with a fixed time horizon $H$. Such an MDP $\mdp$ is characterized through its dynamics and rewards distributions at each round $h\in [H]$. We denote by $p_{\mdp,h}(s,a,s')$ the probability to move to state $s'$ in step $h$ given that the previous state is $s$ and that action $a$ is selected. The mean reward corresponding to this (state, action) pair at this step is denoted by $r_{\mdp,h}(s,a)$. We assume that $\mdp$ is linear in the sense that: for all $h\in [H]$,  
\begin{align}
\forall (s,s')\in {\cal S}, a\in {\cal A}, \quad p_{\mdp,h}(s,a,s') = \phi(s,a)^\top \mu_{\mdp,h}(s') \quad\hbox{and}\quad r_{\mdp,h}(s,a) = \phi(s,a)^\top \theta_{\mdp,h},
\end{align}
where $\mu_{\mdp,h}$ is a family of $d$ measures over $\cS$ seen as a $\cS\times d$ dimensional matrix, and $\theta_{\mdp,h}\in\R^d$. We will assume that $\|\theta_{\mdp,h}\| \le \sqrt{d}$ and $\big\|\sum_{s\in\cS}|\mu_{\mdp,h}(s)|\big\| \le \sqrt{d}$.
A control policy $\pi$ maps, in each step, states to actions. The performance of a policy is captured through its state value functions and its (state, action) value functions respectively defined by: $\forall s\in\cS, a\in \cA, h\in [H]$, 
\begin{align*}
V_{\mdp,h}^\pi(s) &= \E_\mdp\left[\sum_{k=h}^H r_{\mdp,k}(s_k^\pi,\pi(s_k^\pi)) | s_h^\pi=s \right],\\  
Q_{\mdp,h}^\pi(s,a) &= r_{\mdp,h}(s,a) + \sum_{s'} p_{\mdp,h}(s,a,s') V_{\mdp,h+1}^\pi(s').
\end{align*}
The definitions of optimal policies, $\varepsilon$-optimal policies and alternative MDP are the same as in the discounted setup, but only with regard to the value functions at the first step $V_{\mdp,1}^\pi$ and $Q_{\mdp,1}^\pi$ only (we are not interested in $V_{\mdp,h}$ and $Q_{\mdp,h}^\pi$ for $h>1$). The gap however is defined considering all steps, that is
\begin{equation}
    \Delta(\mdp) = \min_{h,s,a\neq\pi_\mdp^\star(s)} V_{\mdp,h}^\star(s) - Q_{\mdp,h}^\star(s,a).
\end{equation}

\subsection{Sample complexity lower bounds}

\begin{prop}
Let $\delta \in (0,1)$ and $\varepsilon > 0$. The sample complexity $\tau$ of any $(\delta,\varepsilon)$-PAC algorithm satisfies $\E_\mdp[\tau] \ge T^\star(\mdp)\textup{kl}(\delta,1-\delta)$ where
\begin{align}
T^\star(\mdp)^{-1} = \sup_{\omega\in(\Sigma_{\cS\times\cA})^H} \inf_{\mdpb\in\textup{Alt}(\mdp)} \sum_{h,s,a} \omega_{h,s,a}\KL(h,s,a).
\end{align}
\end{prop}

\begin{thm}
\label{thm:UpperBoundEpisodic}
We have for all $\omega\in(\Sigma_{\cS\times\cA})^H$,
\begin{equation}
    T(\mdp,\omega)
    \le \frac{10H^2\sum_{h=1}^H\sigma(\omega_h)}{3(\Delta(\mdp)+\varepsilon)^2} = U(\mdp,\omega).
\end{equation}
In particular maximizing over $\omega$ yields
\begin{equation}
    T^\star(\mdp)
    \le \frac{10H^3d}{3(\Delta(\mdp)+\varepsilon)^2} = U^\star(\mdp).
\end{equation}
\end{thm}

\subsection{The GSS-E algorithm}

Next we adapt the GSS algorithm to the episodic setting. The new algorithm is referred to as GSS-E ('E' stands for episodic), and its pseudo-code is presented in Algorithm \ref{alg:gsse}. the sampling strategy in GSS-E is the same as in GSS at each step. Specifically, GSS-E selects in each step $h\in[H]$, a (state, action) pair according to the allocation $\omega^\star$. The same pair is selected in each step; and hence, the realized allocation $\omega_t$ is the same at each step. 
For this reason from now on in the episodic setting and for any allocation $\omega\in\Sigma_{\cS\times\cA}$, $T(\cdot,\omega)$ and $U(\cdot,\omega)$ will denote the corresponding functions with the same allocation $\omega$ for every step.

\medskip

\begin{algorithm*}[H]\label{alg:gsse}
\SetAlgoLined
\DontPrintSemicolon
Compute $\omega^\star = \argmin{\omega\in\Sigma_{\cS\times\cA}}\sigma(\omega)$\;
\While{$Z(t) \le H\beta(\delta/H,t)$}{
    Choose randomly $(s_t,a_t)$ according to $\omega^\star$ and sample this pair for each step $h\in[H]$\;
    Update $\thetath$ and $\muth$ According to \eqref{eq:ThetaLSE} and \eqref{eq:MuLSE}\;
    $t = t+1$\;
    }
    \Return $\hat{\pi} = \pi_t^\star$ the optimal policy of $\mdpt$\;
\caption{The GSS-E algorithm}
\end{algorithm*}

\medskip

As in the discounted setting, the parameters of the MDP are inferred using the Least-Squares Estimators. We can analyzed the error made by these estimators as previously and get a result analogous to proposition \ref{prop:ConcentrationLSEDiscounted}:
\begin{prop}
\label{prop:ConcentrationLSEEpisodic}
Let $\delta\in(0,1)$ and $h\in[H]$. Regardless of the sampling rule, we have with probability at least $1-\delta$ that for all $t\ge1$,
\begin{equation}
    \distMDPtH{\VtH{h+1}^\star}_{t\Lambda(\omega_t)}^2 \le 2H^2 \left(2\log\left(\frac{\sqrt{e}\zeta(2)t^2}{\delta}\right) + d\log\left(8e^4dt^2\right)\right).
\end{equation}
\end{prop}

GSS-E applies the stopping rule defined by:
\begin{equation}
    \tau = \inf\left\{t\ge1 : Z(t) > H\beta(\delta/H,t)\right\},
\end{equation}
The additional $H$ terms in the threshold come from the fact that we have $H$ LSE running simultaneously.
The performance of GSS-E is summarized in the following theorem.

\begin{thm}\label{thm:SampleComplexityEpisodic}
Let $\delta\in(0,1)$. Under the GSS-E algorithm, we have:
$\proba\left[\tau<+\infty, \hat{\pi} \notin \Pi_\varepsilon^\star(\mdp)\right] \le \delta$. Furthermore, 
\begin{align}
    \E[\tau]
    &= \tilde{O}\left(\frac{dH^4}{(\Delta(\mdp)+\varepsilon)^2} \left(\log\left(\frac{1}{\delta}\right) + d\log\left(\frac{d^2H^4}{(\Delta(\mdp)+\varepsilon)^2}\right)\right)\right).
\end{align}
\end{thm}

\begin{proof}[Proof of the first statement of Theorem \ref{thm:SampleComplexityEpisodic}]
The proof of theorem \ref{thm:UpperBoundEpisodic} presents an intermediate bound
\begin{align*}
    U\big(\mdpt,\omega_t\big)^{-1}
    \le \inf_{\mdpb : \pi_t^\star \notin \Pi_\varepsilon^\star(\mdpb)}\frac{6}{5H^2}\sum_{h=1}^H \left\|\thetath-\theta_{\mdpb,h} + (\muth-\mu_{\mdpb,h})^\top \VtH{h+1}^\star\right\|_{\Lambda(\omega_t)}^2
    \le T\big(\mdpt,\omega_t\big)^{-1}.
\end{align*}
Under the event that $\pi_t^\star \notin \Pi_\varepsilon^\star(\mdp)$, we can then write
\begin{align*}
    Z(t) = t\,U\big(\mdpt,\omega_t\big)^{-1}
    \le \frac{6}{5H^2}\sum_{h=1}^H \distMDPtH{\VtH{h+1}^\star}_{t\Lambda(\omega_t)}^2.
\end{align*}
It follows that
\begin{align*}
    \PP\left[\tau<+\infty, \hat{\pi} \notin \Pi_\varepsilon^\star(\mdp)\right]
    &= \PP\left[\exists t\ge1 : Z(t) > H\beta(\delta/H,t), \pi_t^\star \notin \Pi_\varepsilon^\star(\mdp)\right]\\
    &\le \PP\left[\exists t\ge1 : \sum_{h=1}^H \distMDPtH{\VtH{h+1}^\star}_{t\Lambda(\omega_t)}^2 > \frac{5H^3}{6}\beta(\delta/H,t)\right]\\
    &\le \sum_{h=1}^H \PP\left[\exists t\ge1 : \distMDPtH{\VtH{h+1}^\star}_{t\Lambda(\omega_t)}^2 > \frac{5H^2}{6}\beta(\delta/H,t)\right].
\end{align*}
The fact that each terms inside the sum is bounded by $\frac{\delta}{H}$ is exactly the statement of proposition \ref{prop:ConcentrationLSEEpisodic}.
\end{proof}

The proof of the upper bound on the sample complexity is deferred to appendix \ref{app:SampleComplexityEpisodic}.

\newpage
\printbibliography

\newpage
\appendix

\section{Sample Complexity Lower bounds}
\label{app:LowerBound}

\subsection{Proof of Proposition \ref{prop:LowerBoundDiscounted}}
The proof follows a standard change of measure argument to obtain instance specific sample complexity lower bounds (see \cite{kaufmann2016complexity, al2021navigating} and references therein).

\subsection{Gap bounds and value difference lemmas}

Here we present key difference lemmas which are useful to relax the optimization problem that appears in the lower bound.

\begin{lemma}
\label{lemma:GapBound} Let $\varepsilon > 0$ and $\mdpb$ a MDP such that $\pi_\mdp^\star \notin \Pi_\varepsilon^\star(\mdpb)$. Then, we have: 
\begin{itemize}
    \item [(i)] In the discounted setting, it holds that 
            \begin{equation}
            \Delta(\mdp) + \varepsilon
            \le \|V_\mdp^\star-V_\mdpb^{\pi_\mdp^\star}\|_\infty + \|Q_\mdp^\star-Q_\mdpb^\star\|_\infty.
            \end{equation}
    \item [(ii)] In the episodic setting, there exists $h \in [H]$ such that the following holds  
        \begin{equation}
        \Delta(\mdp) + \varepsilon
        \le \|V_{\mdp,h}^\star-V_{\mdpb,h}^{\pi_\mdp^\star}\|_\infty + \|Q_{\mdp,h}^\star-Q_{\mdpb,h}^\star\|_\infty.
        \end{equation}
\end{itemize}
\end{lemma}

\begin{proof}[Proof of Lemma \ref{lemma:GapBound}]
We present the proofs of \emph{(i)} and \emph{(ii)} separately.

\medskip 

\emph{Discounted setting - proof of (i).} $\pi_\mdp^\star \notin \Pi_\varepsilon^\star(\mdpb)$ implies that $\varepsilon \le \max_{s\in\cS} V_\mdpb^\star(s)-V_\mdpb^{\pi_\mdp^\star}(s)$.
Denote $s$ the state maximizing this quantity. We have $\pi_\mdpb^\star(s) \neq \pi_\mdp^\star(s)$.
Indeed if it was not the case then
\begin{align*}
    V_\mdpb^\star(s)-V_\mdpb^{\pi_\mdp^\star}(s)
    & = Q_\mdpb^\star(s,\pi_\mdpb^\star(s))-Q_\mdpb^{\pi_\mdp^\star}(s,\pi_\mdpb^\star(s)) \\
    & = \gamma p_\mdpb(s,\pi_\mdpb^\star(s))^\top (V_\mdpb^\star-V_\mdpb^{\pi_\mdp^\star})\\
    &\le \gamma \max_{s'\in\cS} (V_\mdpb^\star(s')-V_\mdpb^{\pi_\mdp^\star}(s'))\\
    &= \gamma (V_\mdpb^\star(s)-V_\mdpb^{\pi_\mdp^\star}(s))
\end{align*}
which is a contradiction since $\gamma<1$. Now, since $\pi_\mdpb^\star(s) \neq \pi_\mdp^\star(s)$, we have $\Delta(\mdp) \le V_\mdp^\star(s)-Q_\mdp^\star(s,\pi_\mdpb^\star(s))$.
We can then write
\begin{align*}
    \Delta(\mdp) + \varepsilon
    &\le V_\mdp^\star(s)-Q_\mdp^\star(s,\pi_\mdpb^\star(s)) + V_\mdpb^\star(s)-V_\mdpb^{\pi_\mdp^\star}(s)\\
    &= V_\mdp^\star(s)-V_\mdpb^{\pi_\mdp^\star}(s) + Q_\mdpb^\star(s,\pi_\mdpb^\star(s))-Q_\mdp^\star(s,\pi_\mdpb^\star(s))\\
    &\le \|V_\mdp^\star-V_\mdpb^{\pi_\mdp^\star}\|_\infty + \|Q_\mdp^\star-Q_\mdpb^\star\|_\infty.
\end{align*}

\medskip 

\emph{Episodic setting - proof of (ii).} For each step $h\in[H]$ we denote $s_h = \arg\max_s V_{\mdpb,h}^\star(s)-V_{\mdpb,h}^{\pi_\mdp^\star}(s)$.
Since $\pi_\mdp^\star\notin\Pi_\varepsilon^\star(\mdpb)$, we have $V_{\mdpb,1}^\star(s_1)-V_{\mdpb,1}^{\pi_\mdp^\star}(s_1) \ge \varepsilon$.
Note that if $\pi_{\mdpb,1}^\star(s_1) = \pi_{\mdp,1}^\star(s_1)$ then
\begin{align*}
    V_{\mdpb,1}^\star(s_1)-V_{\mdpb,1}^{\pi_\mdp^\star}(s_1)
    = p_{\mdpb,1}(s_1,\pi_{\mdpb,1}^\star(s_1))^\top (V_{\mdpb,2}^\star-V_{\mdpb,2}^{\pi_\mdp^\star})
    \le V_{\mdpb,2}^\star(s_2)-V_{\mdpb,2}^{\pi_\mdp^\star}(s_2).
\end{align*}
Iterating this reasoning we can show that there exists a step $h$ such that $\pi_{\mdpb,h}^\star(s_h) \neq \pi_{\mdp,h}^\star(s_h)$ (else we would end up with $\varepsilon \le 0$) and that $\varepsilon \le V_{\mdpb,h}^\star(s_h)-V_{\mdpb,h}^{\pi_\mdp^\star}(s_h)$.
We can then write
\begin{align*}
    \Delta(\mdp) + \varepsilon
    &\le V_{\mdp,h}^\star(s_h)-Q_{\mdp,h}^\star(s,\pi_{\mdpb,h}^\star(s_h)) + V_{\mdpb,h}^\star(s_h)-V_{\mdpb,h}^{\pi_\mdp^\star}(s_h)\\
    &= V_{\mdp,h}^\star(s_h)-V_{\mdpb,h}^{\pi_\mdp^\star}(s_h) + Q_{\mdpb,h}^\star(s_h,\pi_{\mdpb,h}^\star(s_h))-Q_{\mdp,h}^\star(s,\pi_{\mdpb,h}^\star(s_h))\\
    &\le \|V_{\mdp,h}^\star-V_{\mdpb,h}^{\pi_\mdp^\star}\|_\infty + \|Q_{\mdp,h}^\star-Q_{\mdpb,h}^\star\|_\infty.
\end{align*}

\end{proof}

\newpage

\begin{lemma}
\label{lemma:ValueDiffSamePolicy}
Let $\pi$ be any deterministic policy. We have:
\begin{itemize}
    \item [(i)] In the discounted setting, it holds that
    \begin{equation}
    \label{eq:ValueDiffSamePolicyDiscounted}
    \|V_\mdp^\pi-V_\mdpb^\pi\|_\infty
    \le \|Q_\mdp^\pi-Q_\mdpb^\pi\|_\infty
    \le \frac{1}{1-\gamma} \max_{s,a} \left|\phi(s,a)^\top\left(\diffMDPb{V_\mdp^\pi}\right)\right|.
    \end{equation}
    \item [(ii)] In the episodic setting, it holds for all $h_0\in[H]$ that
\begin{equation}
    \label{eq:ValueDiffSamePolicyEpisodic}
    \|V_{\mdp,h_0}^\pi-V_{\mdpb,h_0}^\pi\|_\infty
    \le \|Q_{\mdp,h_0}^\pi-Q_{\mdpb,h_0}^\pi\|_\infty
    \le \sum_{h=h_0}^H \max_{s,a} \left|\phi(s,a)^\top\left(\diffMDPbH{V_{\mdp,h}^\pi}\right)\right|.
\end{equation} 
\end{itemize}

\end{lemma}

\begin{proof}[Proof of Lemma \ref{lemma:ValueDiffSamePolicy}]
We present the proofs of \emph{(i)} and \emph{(ii)} separately.

\medskip 

\emph{Discounted setting - Proof of (i).}  For any $s\in\cS$ we have $V_\mdp^\pi(s)-V_\mdpb^\pi(s) = Q_\mdp^\pi(s,\pi(s))-Q_\mdpb^\pi(s,\pi(s))$ thus the first inequality.
Now, we can write for any pair $(s,a)\in\cS\times\cA$
\begin{align*}
    Q_\mdp^\pi(s,a)-Q_\mdpb^\pi(s,a)
    &= \phi(s,a)^\top\left(\theta_\mdp-\theta_\mdpb + \gamma(\mu_\mdp-\mu_\mdpb)^\top V_\mdp^\pi\right) + \gamma p_\mdpb(s,a)^\top(V_\mdp^\pi-V_\mdpb^\pi),
\end{align*}
so that
\begin{align*}
    \|Q_\mdp^\pi-Q_\mdpb^\pi\|_\infty
    &\le \max_{s,a} \big|\phi(s,a)^\top\big(\theta_\mdp-\theta_\mdpb + \gamma(\mu_\mdp-\mu_\mdpb)^\top V_\mdp^\pi\big)\big| + \gamma\|V_\mdp^\pi-V_\mdpb^\pi\|_\infty\\
    &\le \max_{s,a} \big|\phi(s,a)^\top\big(\theta_\mdp-\theta_\mdpb + \gamma(\mu_\mdp-\mu_\mdpb)^\top V_\mdp^\pi\big)\big| + \gamma\|Q_\mdp^\pi-Q_\mdpb^\pi\|_\infty,
\end{align*}
which implies the second inequality. 

\medskip 

\emph{Episodic setting - Proof of (ii).} It is immediate that for any $h\in[H]$ $\|V_{\mdp,h}^\pi-V_{\mdpb,h}^\pi\|_\infty \le \|Q_{\mdp,h}^\pi-Q_{\mdpb,h}^\pi\|_\infty$.
Now, as in \emph{(i)} we can write for any $h$
\begin{align*}
    \|Q_{\mdp,h}^\pi-Q_{\mdpb,h}^\pi\|_\infty
    \le \max_{s,a} \big|\phi(s,a)^\top\big(\diffMDPbH{V_{\mdp,h+1}^\pi}\big)\big| + \|Q_{\mdp,h+1}^\pi-Q_{\mdpb,h+1}^\pi\|_\infty
\end{align*}
and conclude by iterating this inequality for $h=h_0$ to $H$.
\end{proof}

\begin{lemma}
\label{lemma:ValueDiffOptimalPolicy}
We have
\begin{itemize}
    \item [(i)] In the discounted setting, it holds that 
    \begin{equation}
    \|V_\mdp^\star-V_\mdpb^\star\|_\infty
    \le \|Q_\mdp^\star-Q_\mdpb^\star\|_\infty
    \le \frac{1}{1-\gamma} \max_{s,a} \left|\phi(s,a)^\top\left(\diffMDPb{V_\mdp^\star}\right)\right|.
\end{equation}
    \item [(ii)] In the episodic setting, it holds for all $h_0 \in [H]$ that 
    \begin{equation}
        \|V_{\mdp,h_0}^\star-V_{\mdpb,h_0}^\star\|_\infty
        \le \|Q_{\mdp,h_0}^\star-Q_{\mdpb,h_0}^\star\|_\infty
        \le \sum_{h=h_0}^H \max_{s,a} \left|\phi(s,a)^\top\left(\diffMDPbH{V_{\mdp,h}^\star}\right)\right|.
    \end{equation}
\end{itemize}
\end{lemma}

\begin{proof}[Proof of Lemma \ref{lemma:ValueDiffOptimalPolicy}] We present the proof of \emph{(i)} and \emph{(ii)} separately.

\medskip 

\emph{Discounted setting - Proof of (i).} Let $s\in\cS$. We have by optimality of $\pi_\mdpb^\star$ that 
\begin{align*}
    V_\mdp^\star(s)-V_\mdpb^\star(s)
    &= Q_\mdp^\star(s,\pi_\mdp^\star(s))-Q_\mdpb^\star(s,\pi_\mdpb^\star(s))\\
    &\le Q_\mdp^\star(s,\pi_\mdp^\star(s))-Q_\mdpb^\star(s,\pi_\mdp^\star(s))\\
    &\le \|Q_\mdp^\star-Q_\mdpb^\star\|_\infty.
\end{align*}
$V_\mdpb^\star(s)-V_\mdp^\star(s)$ can be bounded the same way using the optimality of $\pi_\mdp^\star$, so that this inequality is true in absolute value which gives the first inequality.
Now, we can write for any pair $(s,a)\in\cS\times\cA$
\begin{align*}
    Q_\mdp^\star(s,a)-Q_\mdpb^\star(s,a)
    = \phi(s,a)^\top\left(\diffMDPb{V_\mdp^\star}\right) + \gamma p_\mdpb(s,a)^\top(V_\mdp^\star-V_\mdpb^\star),
\end{align*}
so that
\begin{align*}
    \|Q_\mdp^\star-Q_\mdpb^\star\|_\infty
    &\le \max_{s,a} \left|\phi(s,a)^\top\left(\diffMDPb{V_\mdp^\star}\right)\right| + \gamma \|V_\mdp^\star-V_\mdpb^\star\|_\infty\\
    &\le \max_{s,a} \left|\phi(s,a)^\top\left(\diffMDPb{V_\mdp^\star}\right)\right| + \gamma \|Q_\mdp^\star-Q_\mdpb^\star\|_\infty
\end{align*}
which implies the result.

\medskip 

\emph{Episodic setting - proof of (ii).} For any $h\in[H]$ we can write with the same reasoning as in the proof of \emph{(i)} that  $\|V_{\mdp,h}^\star-V_{\mdpb,h}^\star\|_\infty \le \|Q_{\mdp,h}^\star-Q_{\mdpb,h}^\star\|_\infty$ and
\begin{align*}
    \|Q_{\mdp,h}^\star-Q_{\mdpb,h}^\star\|_\infty
    \le \max_{s,a} \big|\phi(s,a)^\top\big(\diffMDPbH{V_{\mdp,h+1}^\star}\big)\big| + \|Q_{\mdp,h+1}^\star-Q_{\mdpb,h+1}^\star\|_\infty
\end{align*}
and conclude by iterating this inequality for $h=h_0$ to $H$.
\end{proof}


\begin{remark}
In lemmas \ref{lemma:ValueDiffSamePolicy} and \ref{lemma:ValueDiffOptimalPolicy} we have used the fact that for any $(s,a)$, $\|p_\mdpb(s,a)\|_1 = 1$, but only with $\mdpb$ and not with $\mdp$.
When working with the LSE estimators $\thetat$ and $\mut$ we will construct a MDP $\mdpt$ which transitions probabilities, defined as $\phi(s,a)^\top\mu_t^\top$, may not be actual probability vectors. This is not an issue since these lemmas will only be used with $\mdpt$ taking the place of the first MDP which does not require such property.
\end{remark}

\subsection{Proof of Theorem \ref{thm:UpperBoundDiscounted}}
The goal of this section is to show the bound $T(\mdp,\omega) \le U(\mdp,\omega)$ for a given MDP $\mdp$ and a given allocation $\omega$. In other words, the goal is to show that
\begin{align*}
    T(\mdp,\omega)^{-1} = \inf_{\mdpb\in\Alt_\varepsilon(\mdp)} \sum_{(s,a)\in\cS\times\cA}\omega_{s,a}\KL(s,a)
    \ge \frac{3(1-\gamma)^4(\Delta(\mdp)+\varepsilon)^2}{10\sigma(\omega)}.
\end{align*}
We are actually going to show this bound but with an infimum over the set of MDPs $\mdpb$ such that $\pi_\mdp^\star \notin \Pi_\varepsilon^\star(\mdpb)$, which is larger than $\Alt_\varepsilon(\mdp)$ and thus gives a smaller infimum than $T(\mdp,\omega)^{-1}$. From now on we consider one such MDP $\mdpb$.
The Kullback-Leibler divergence can be lower bounded using lemma \ref{lemma:BoundKL}. For a given pair $(s,a)$, we choose $f=r+\gamma V_\mdp^\star(s')$ where $r$ and $s'$ are respectively the random reward and the random next step after playing the pair $(s,a)$. $f$ is almost surely bounded by $(1-\gamma)^{-1}$ and the lemma gives
\begin{align*}
    \KL(s,a)
    &\ge \frac{6(1-\gamma)^2}{5} \left(\E_{\mdp(s,a)}[r + \gamma V_\mdp^\star(s')] - \E_{\mdpb(s,a)}[r + \gamma V_\mdp^\star(s')]\right)^2\\
    &= \frac{6(1-\gamma)^2}{5} \left(\phi(s,a)^\top\left(\diffMDPb{V_\mdp^\star}\right)\right)^2.
\end{align*}
Summing over all state action pairs,
\begin{align}
\label{eq:IntermediateLowerBoundDiscounted}
    \sum_{(s,a)\in\cS\times\cA}\omega_{s,a}\KL(s,a)
    &\ge \frac{6(1-\gamma)^2}{5} \distMDPb{V_\mdp^\star}_{\Lambda(\omega)}^2
\end{align}

Putting together Lemma \ref{lemma:GapBound}, Lemma  \ref{lemma:ValueDiffSamePolicy} and Lemma \ref{lemma:ValueDiffOptimalPolicy} (and choosing $\pi=\pi_\mdp^\star$ in Lemma \ref{lemma:ValueDiffSamePolicy}), obtain a bound on the quantity $\Delta(\mdp)+\varepsilon$ as follows 
\begin{align*}
    \Delta(\mdp)+\varepsilon
    &\le \frac{2}{1-\gamma} \max_{s,a} \left|\phi(s,a)^\top\left(\diffMDPb{V_\mdp^\star}\right)\right|.
\end{align*}
Now, we can apply lemma \ref{lemma:Optimization} with $n=1$, $\Delta = \frac{1-\gamma}{2}(\Delta(\mdp)+\varepsilon)$, $\Lambda_1=\Lambda(\omega)$ and $\phi_1$ the feature maximizing the term above, and deduce that
\begin{align*}
    \left\|\diffMDPb{V_\mdp^\star}\right\|_{\Lambda(\omega)}^2
    &\ge \frac{(1-\gamma)^2(\Delta(\mdp)+\varepsilon)^2}{4\|\phi\|_{\Lambda(\omega)^{-1}}^2}\\
    &\ge \frac{(1-\gamma)^2(\Delta(\mdp)+\varepsilon)^2}{4\sigma(\omega)}.
\end{align*}
Putting this together with equation \eqref{eq:IntermediateLowerBoundEpisodic} and then taking the infimum over $\mdpb$, we have
\begin{equation}\label{eq:ineq1}
    T(\mdp,\omega)^{-1}
    \ge \frac{3(1-\gamma)^4(\Delta(\mdp)+\varepsilon)^2}{10\sigma(\omega)}.
\end{equation}
Now, optimizing over $\omega \in \Sigma_{S \times A}$, we obtain that 
\begin{align*}
    T^\star(\cM) = \inf_{\omega \in \Sigma_{S \times A}} T(\cM, \omega)  \le \frac{10 \inf_{\omega \in \Sigma_{S \times A}}\sigma(\omega)}{3(1-\gamma)^4 (\Delta(\cM) + \varepsilon^2)} = U^\star(\cM)
\end{align*}
Now, applying Kiefer-Wolfowitz theorem (see Theorem \ref{thm:kiefer-wolfowitz}) entails that $\inf_{\omega \in \Sigma_{S \times A}}\sigma(\omega) = d$, and that $\omega^\star(\cM)$ which achieves the minimum is the so-called G-optimal design (see \cite{lattimore2020bandit} and references therein). This concludes the proof of Theorem \ref{thm:UpperBoundDiscounted}.

\begin{thm}[Kiefer-Wolfowitz \cite{kiefer1960equivalence}] \label{thm:kiefer-wolfowitz} Let $\Phi \subseteq \RR^d$ be a finite set and $\textup{span}(\Phi) = d$. Let $\Sigma$ be the set of probability distributions supported on $\Phi$, then the following statements are equivalent:
\begin{itemize}
    \item [(i)] $\omega^\star = \arg\min_{\omega \in \Sigma}\max_{\phi \in \Phi} \phi^\top (\sum_{\phi \in \Phi}\omega(\phi) \phi \phi^\top )^{-1} \phi$,
    \item [(ii)] $\omega^\star = \arg\max_{\omega \in \Sigma} \log\det(\sum_{\phi \in \Phi}\omega(\phi) \phi \phi^\top )$,
    \item [(iii)] $\max_{\phi \in \Phi} \phi^\top (\sum_{\phi \in \Phi}\omega^\star(\phi) \phi \phi^\top )^{-1} \phi = d$.
\end{itemize}
\end{thm}

\begin{remark} The statement of the Kiefer-Wolfowitz theorem in \cite{kiefer1960equivalence} holds under a much weaker assumption than that of a finite set $\Phi$. For example, if $\Phi = \lbrace \phi(x): x \in \cX \rbrace$ where $\phi:\cX \to \RR^d$ is a continuous map on some compact set $\cX$, then the equivalence between the three statements \emph{(i), (ii)} and \emph{(iii)} still holds.
\end{remark}

\subsection{Proof of Theorem \ref{thm:UpperBoundEpisodic}}

Our goal is to show that
\begin{align*}
    T(\mdp,\omega)^{-1} = \inf_{\mdpb\in\Alt_\varepsilon(\mdp)} \sum_{h=1}^H \sum_{(s,a)\in\cS\times\cA}\omega_{h,s,a}\KL(h,s,a)
    \ge \frac{3(\Delta(\mdp)+\varepsilon)^2}{10H^2\sum_{h=1}^H\sigma(\omega_h)}.
\end{align*}
We are actually going to show this bound but with an infimum over the set of MDPs $\mdpb$ such that $\pi_\mdp^\star \notin \Pi_\varepsilon^\star(\mdpb)$, which is larger than $\Alt_\varepsilon(\mdp)$ and thus gives a smaller infimum than $T(\mdp,\omega)^{-1}$. From now on we consider one such MDP $\mdpb$.
The term $\KL(h,s,a)$ can be lower bounded using lemma \ref{lemma:BoundKL} like we did in the discounted model by choosing the function $f=r+\gamma V_{\mdp,h+1}^\star(s')$ where $r$ and $s'$ are respectively the random reward and the random next step after playing the pair $(s,a)$ at step $h$. $f$ is almost surely bounded by $H$ and the lemma gives
\begin{align}
\label{eq:IntermediateLowerBoundEpisodic}
    \sum_{h=1}^H \sum_{s,a}\omega_{s,a} \KL(s,a)
    &\ge \frac{6}{5H^2} \sum_{h=1}^H \sum_{s,a}\omega_{s,a} \left(\phi(s,a)^\top\left(\diffMDPbH{V_{\mdp,h+1}^\star}\right)\right)^2\\
    &= \frac{6}{5H^2} \sum_{h=1}^H \distMDPbH{V_{\mdp,h+1}^\star}_{\Lambda(\omega_h)}^2.
\end{align}

Putting together Lemma \ref{lemma:GapBound}, Lemma  \ref{lemma:ValueDiffSamePolicy} and Lemma \ref{lemma:ValueDiffOptimalPolicy} (and choosing $\pi=\pi_\mdp^\star$ in Lemma \ref{lemma:ValueDiffSamePolicy}), obtain a bound on the quantity $\Delta(\mdp)+\varepsilon$ as follows 
\begin{align*}
    \Delta(\mdp)+\varepsilon
    &\le 2\sum_{h=1}^H \max_{s,a} \left|\phi(s,a)^\top\left(\diffMDPbH{V_{\mdp,h+1}^\star}\right)\right|.
\end{align*}
Now, we can apply lemma \ref{lemma:Optimization} with $n=H$, $\Delta = \frac{1}{2}(\Delta(\mdp)+\varepsilon)$, $\Lambda_h=\Lambda(\omega_h)$ and $\phi_h$ the feature maximizing the $h$-th term in the sum above, and deduce that
\begin{align*}
    \sum_{h=1}^H \left\|\diffMDPbH{V_{\mdp,h+1}^\star}\right\|_{\Lambda(\omega_h)}^2
    &\ge \frac{(\Delta(\mdp)+\varepsilon)^2}{4\sum_{h=1}^H\|\phi_h\|_{\Lambda(\omega_h)^{-1}}^2}\\
    &\ge \frac{(\Delta(\mdp)+\varepsilon)^2}{4\sum_{h=1}^H\sigma(\omega_h)}.
\end{align*}
Putting this together with equation \eqref{eq:IntermediateLowerBoundEpisodic} and then taking the infimum over $\mdpb$, we have
\begin{equation}
    T(\mdp,\omega)^{-1}
    \ge \frac{3(\Delta(\mdp)+\varepsilon)^2}{10H^2\sum_{h=1}^H\sigma(\omega_h)}.
\end{equation}

\subsection{Technical lemmas}

\begin{lemma}
\label{lemma:Optimization}
Let $(\phi_i)_i\in(\R^d)^n$, $\Delta>0$ and $(\Lambda_i)_i\in(\R^{d\times d})^n$ some definite symmetric matrices.
We have the following optimisation problem exact solution
\begin{equation}
    \underset{\sum_{i=1}^n |\phi_i^\top x_i| \ge \Delta}{\inf_{x\in\R^{n\times d}}} \sum_{i=1}^n\|x_i\|_{\Lambda_i}^2
    = \frac{\Delta^2}{\sum_{i=1}^n\|\phi_i\|_{\Lambda_i^{-1}}^2}.
\end{equation}
\end{lemma}

\begin{proof}[Proof of Lemma \ref{lemma:Optimization}]
The absolute values can be removed from the constraint $\sum_i|\phi_i^\top x_i| \ge \Delta$, as we can then apply it adding arbitrary signs before each $\phi_i$ and get the same result since $\|-\phi_i\|_{\Lambda_i^{-1}} = \|\phi_i\|_{\Lambda_i^{-1}}$.
The Lagrangian of the problem without the absolute value is
\begin{equation*}
    \cL(x,\nu) = \sum_{i=1}^n \|x_i\|_{\Lambda_i}^2 - \nu\left(\sum_{i=1}^n \phi_i^\top x_i - \Delta\right)
\end{equation*}
and the KKT conditions for optimality are
\begin{align*}
    & \forall i,\; 2\Lambda_i x_i - \nu\phi_i = 0,\\
    & \nu\left(\Delta - \sum_{i=1}^n \phi_i^\top x_i\right) = 0,\\
    & \Delta \le \sum_{i=1}^n \phi_i^\top x_i,\\
    & \nu \ge 0.
\end{align*}
The first one gives $2x_i = \nu\Lambda_i^{-1}\phi_i$.
This formula together with the third condition imply that $\nu>0$, so that the third condition is an equality and
\begin{align*}
    \nu = \frac{2\Delta}{\sum_{i=1}^n \phi_i^\top\Lambda_i^{-1}\phi_i}
    = \frac{2\Delta}{\sum_{i=1}^n \|\phi_i\|_{\Lambda_i^{-1}}^2}
\end{align*}
Finally we have
\begin{align*}
    x_i = \Delta\cdot \frac{\Lambda_i^{-1}\phi_i}{\sum_{i=1}^n \|\phi_i\|_{\Lambda_i^{-1}}^2}
\end{align*}
and the solution of the problem is
\begin{align*}
    \frac{\sum_{i=1}^n \phi_i^\top\Lambda_i^{-1}\Lambda_i\Lambda_i^{-1}\phi_i}{\left(\sum_{i=1}^n \|\phi_i\|_{\Lambda_i^{-1}}^2\right)^2} \Delta^2
    = \frac{\Delta^2}{\sum_{i=1}^n \|\phi_i\|_{\Lambda_i^{-1}}^2}.
\end{align*}
\end{proof}

\begin{lemma}
\label{lemma:BoundKL}
Let $\alpha$ and $\beta$ be two probability measures and $f$ a random bounded variable such that $f \ge 0$.
Then we have the following inequality :
\begin{equation}
    \textup{KL}(\alpha\|\beta) \ge \frac{6}{5\|f\|_\infty^2}(\E_\alpha[f] - \E_\beta[f])^2.
\end{equation}
\end{lemma}

\begin{proof}
We prove that if $\E_\beta[f]=0$ then
\begin{equation*}
    \textup{KL}(\alpha\|\beta) \ge \frac{6}{5\|f\|_\infty^2}\E_\alpha[f]^2.
\end{equation*}
It then suffices to apply this result to $f-\E_\beta[f]$ and to notice that if $f\ge0$ then $\|f-\E_\beta[f]\|_\infty \le \|f\|_\infty$.

Let $f$ be centered with regard to $\beta$.
Using Donsker-Varadhan's inequality, we know that for any $\lambda>0$,
$$ \textup{KL}(\alpha\|\beta) \ge E_\alpha[\lambda f] - \log(E_\beta[\exp(\lambda f)]).$$
Now,
\begin{align*}
    \E_\beta[\exp(\lambda f)]
    &\le E_\beta\left[1+\lambda f+f^2\sum_{k=2}^{+\infty}\frac{\lambda^k \|f\|_\infty^{k-2}}{k!}\right]\\
    &\le 1 + \frac{\V_\beta[f]}{\|f\|_\infty^2}\left(e^{\lambda \|f\|_\infty}-\lambda \|f\|_\infty - 1\right)\\
    &\le 1 + \frac{1}{4}\left(e^{\lambda \|f\|_\infty} - \lambda \|f\|_\infty - 1\right).
\end{align*}
Using $\log(1+u) \le u$,
$$\textup{KL}(\alpha\|\beta) \ge \E_\alpha[\lambda f] - \frac{1}{4}\left(e^{\lambda \|f\|_\infty} - \lambda \|f\|_\infty - 1\right).$$
Optimizing over $\lambda$ by choosing $\lambda = \frac{1}{\|f\|_\infty}\log\left(1+4\frac{\E_\alpha[f]}{\|f\|_\infty}\right)$, we get
$$\textup{KL}(\alpha\|\beta)
\ge \frac{1}{4} \left(\left(1+4\frac{\E_\alpha[f]}{\|f\|_\infty}\right) \log\left(1+4\frac{\E_\alpha[f]}{\|f\|_\infty}\right) - 4\frac{\E_\alpha[f]}{\|f\|_\infty}\right).$$
Using Bernstein's inequality $(1+u)\log(1+u)-u \ge \frac{u^2}{2(1+u/3)}$, we finally have
$$\textup{KL}(\alpha\|\beta)
\ge \frac{\left(4\frac{\E_\alpha[f]}{\|f\|_\infty}\right)^2}{8\left(1+\frac{4}{3}\frac{\E_\alpha[f]}{\|f\|_\infty}\right)}
= \frac{2\E_\alpha[f]^2}{\|f\|_\infty^2+\frac{2}{3}\|f\|_\infty\E_\alpha[f]}
\ge \frac{6}{5\|f\|_\infty^2}\E_\alpha[f]^2.$$
\end{proof}

\newpage
\section{Sampling rules}

\subsection{Proof of Proposition \ref{prop:ConcentrationLambda}}

\paragraph{Sampling under the G-optimal design.} It may be ambitious to target a sampling allocation that corresponds exactly to the G-optimal design. Instead, we may focus on a solution that is only approximately optimal. We will say that an allocation (or design) $\tilde{\omega}^\star \in \Sigma_{S \times A}$ is an $\epsilon$-approximate G-optimal design if it satisfies 
\begin{align}\label{eq:G-approx}
    \max_{(s,a)\in \cS \times \cA} \Vert \phi(s,a) \Vert_{\Lambda(\tilde{\omega}^\star)^{-1}}^2    \le (1+\epsilon) \inf_{\omega \in \Sigma}\max_{(s,a)\in \cS \times \cA} \Vert \phi(s,a) \Vert_{\Lambda(\omega)^{-1}}^2 = (1+\epsilon) d.
\end{align}
Obtaining such a solution may be obtained efficiently using a Frank-Wolfe algorithm (see \cite{lattimore2020bandit} and references therein). Classically, existing procedures that use G-optimal design as a basis for their sampling schemes, do that in a deterministic fashion by requiring a budget of samples ahead \cite{lattimore2020learning}, or using efficient rounding procedures coupled with a doubling trick \cite{soare2014best}. In our case, we don't use a doubling trick, and our budget of samples is random since we are using a stopping time. Our approach is to directly sample from an approximate G-optimal design, and we shall see that in fact this is sufficient.    

\paragraph{A matrix concentration result.} We will prove a slightly stronger result that is valid for all  $\epsilon$-approximate G-optimal designs. 
   
\begin{lemma}\label{lemma:ConcentrationLambdaGeneral} Let $\tilde{\omega}^\star \in \Sigma_{S\times A}$, be an $\epsilon$-approximate G-optimal design for some $\epsilon > 0$ (i.e., satisfying \eqref{eq:G-approx}). Assume that the sequence of state action pairs $(s_t,a_t)_{t \ge 1}$ are sampled according to $\tilde{\omega}^\star$, then, for all $\delta \in (0,1)$, $\rho > 0$, we have 
\begin{align*}
    \forall t \ge 2(1+\epsilon)\left(\frac{1}{\rho^2}+\frac{1}{3\rho}\right)d \log\left(\frac{2d}{\delta}\right), \qquad \PP\left( (1-\rho)\Lambda(\tilde{\omega}^\star) \preceq \Lambda(\omega_t) \preceq (1+\rho) \Lambda(\tilde{\omega}^\star)  \right) \ge 1-\delta.
\end{align*}
\end{lemma}

\begin{remark}
Note that the statement of Lemma \ref{lemma:ConcentrationLambdaGeneral}, along with the fact that $\tilde{\omega}^\star$ is an $\epsilon$-approximate G-optimal design, ensures that the event 
\begin{align*}
    \frac{d}{1+\rho} \le \max_{(s,a)\in \cS \times \cA }\Vert \phi(s,a)\Vert_{\Lambda(\omega_t)^{-1}}^2 \le \frac{(1+\epsilon)d}{1-\rho}  
\end{align*}
holds with probability at least $1-\delta$, provided $t \ge 2(1+\epsilon)\big(\frac{1}{\rho^2}+\frac{1}{3\rho}\big)d \log\big(\frac{2d}{\delta}\big)$. Note that the maximum over  $\cS\times \cA $ came for free thanks to the matrix concentration, and this concentration did not require a priori any condition on the finiteness of the set $\cS \times \cA$. Actually, the above generalizes immediately for any continuous and compact state-action spaces $\cS \times \cA$, provided we can compute an $\epsilon$-approximate $G$-optimal design. 
\end{remark}

\begin{remark}
In particular, specializing \ref{lemma:ConcentrationLambdaGeneral} to the G-optimal design $\omega^\star$ and choosing $\rho = 1/2$ gives 
\begin{align}
\label{eq:ConcentrationLambda}
    \forall t \ge \frac{28d}{3} \log\left(\frac{2d}{\delta}\right), \qquad
    \PP\left[\sigma(\omega_t) \le 2\sigma(\omega^\star)\right] \ge 1-\delta.
\end{align}
This is exactly the statement of Proposition \ref{prop:ConcentrationLambda}.
\end{remark}

\begin{proof}[Proof of Lemma \ref{lemma:ConcentrationLambdaGeneral}] 
Let $\delta\in(0,1)$ and $t\ge1$. First, we have
\begin{align*}
    (\tilde{\Lambda}^\star)^{-1/2}\Lambda(\omega_t)(\tilde{\Lambda}^\star)^{-1/2} - I_d
    = \sum_{\ell=1}^t  \frac{1}{t}\left(\big((\tilde{\Lambda}^\star)^{-1/2}\phi(s_\ell,a_\ell) \big) \big((\tilde{\Lambda}^\star)^{-1/2}\phi(s_\ell,a_\ell) \big)^\top - I_d\right).
\end{align*}
where we denote $\tilde{\Lambda}^\star = \Lambda(\tilde{\omega}^\star)$. Denote $(X_\ell)_{1\le\ell\le t}$ the summands appearing in the sum above. Note that $X_\ell$ is a symmetric random matrix that satisfies for all $\ell\ge 1$, $\|X_\ell\| \le \frac{(1+\epsilon)d}{t}$ a.s. and $\|\E[X_\ell^2]\| \le \frac{(1+\epsilon)d}{t^2}$ for the operator norm.
Indeed, we have for any $(s,a)\in\cS\times\cA$,
\begin{align*}
    \left\|\left((\tilde{\Lambda}^\star)^{-1/2}\phi(s,a)\right) \left((\tilde{\Lambda}^\star)^{-1/2}\phi(s,a)\right)^{\top}\right\|
    &= \max_{\|x\|=1} x^\top\left((\tilde{\Lambda}^\star)^{-1/2}\phi(s,a)\right) \left((\tilde{\Lambda}^\star)^{-1/2}\phi(s,a)\right)^{\top}x\\
    &= \max_{\|x\|=1} \left(\left((\tilde{\Lambda}^\star)^{-1/2}\phi(s,a)\right)^{\top}x\right)^2\\
    &\le \left\|\phi(s,a) \right\|_{(\tilde{\Lambda}^\star)^{-1} }^2\\
    &\le (1+\epsilon) d
\end{align*}
so that a.s.
\begin{align*}
    \|X_\ell\|
    \le \frac{1}{t} \max\left(\left\|\left((\tilde{\Lambda}^\star)^{-1/2}\phi(s_\ell,a_\ell)\right) \left((\tilde{\Lambda}^\star)^{-1/2}\phi(s_\ell,a_\ell)\right)^{\top}\right\|, \|I_d\|\right)
    \le \frac{(1+\epsilon)d}{t}
\end{align*}
and, since $\E_{(s,a)\sim\tilde{\omega}^\star}\big[\big((\tilde{\Lambda}^\star)^{-1/2}\phi(s,a)\big) \left((\tilde{\Lambda}^\star)^{-1/2}\phi(s,a)\right)^{\top}\big] = (\tilde{\Lambda}^\star)^{-1/2}\tilde{\Lambda}^\star(\tilde{\Lambda}^\star)^{-1/2} = I_d$,
\begin{align*}
    \E[X_\ell^2]
    &\preceq \E_{(s,a)\sim\tilde{\omega}^\star}\left[ \left(\frac{1}{t}\left((\tilde{\Lambda}^\star)^{-1/2}\phi(s,a)\right) \left((\tilde{\Lambda}^\star)^{-1/2}\phi(s,a)\right)^{\top} \right)^2\right]\\
    &\preceq \frac{1}{t^2} \E_{(s,a)\sim\tilde{\omega}^\star}\left[\left((\tilde{\Lambda}^\star)^{-1/2}\phi(s,a)\right) \left((\tilde{\Lambda}^\star)^{-1/2}\phi(s,a)\right)^{\top} \right] \max_{s,a}\left\|\left((\tilde{\Lambda}^\star)^{-1/2}\phi(s,a)\right) \left((\tilde{\Lambda}^\star)^{-1/2}\phi(s,a)\right)^{\top}\right\|\\
    &\preceq \frac{(1+\epsilon)d}{t^2} I_d.
\end{align*}
Now, using Matrix  Bernstein's inequality (more precisely, we use Theorem 5.4.1. in \cite{vershynin2018high}, see also \cite{tropp2015introduction}), we obtain that for all $\rho > 0$,
\begin{align*}
    \proba\left(\left\|\sum_{\ell=1}^tX_\ell\right\|> \rho \right)
    \le 2d\exp\left(-\frac{t\rho^2}{2(1+\epsilon)(1+\rho/3)d}\right).
\end{align*}
which implies that 
\begin{align*}
    \forall t \ge 2(1+\epsilon)\left(\frac{1}{\rho^2}+\frac{1}{3\rho}\right)d \log\left(\frac{2d}{\delta}\right), \qquad \proba\left(\left\|\sum_{\ell=1}^tX_\ell\right\|> \rho \right)
    \le \delta. 
\end{align*}
Finally, in order to conclude, observe that  
\begin{align*}
    \Vert (\tilde{\Lambda}^\star)^{-1/2}\Lambda(\omega_t)(\tilde{\Lambda}^\star)^{-1/2} - I_d  \Vert \le \rho \implies  (1-\rho )\tilde{\Lambda}^\star \preceq  \Lambda(\omega_t) \preceq (1+\rho) \tilde{\Lambda}^\star
\end{align*}
thus, provided $t \ge 2(1+\epsilon)\big(\frac{1}{\rho^2}+\frac{1}{3\rho}\big)d \log\big(\frac{2d}{\delta}\big)$, it follows that   
\begin{align*}
    \PP((1-\rho )\tilde{\Lambda}^\star \preceq  \Lambda(\omega_t) \preceq (1+\rho) \tilde{\Lambda}^\star) \ge \PP\left( \Vert (\tilde{\Lambda}^\star)^{-1/2}\Lambda(\omega_t)(\tilde{\Lambda}^\star)^{-1/2} - I_d  \Vert \le \rho \right) \ge 1-\delta 
\end{align*}.
\end{proof}

\newpage
\section{Least Square Estimation and Stopping Rules}

In this section we show the correctness of our stopping rules through the concentration inequalities of the least square estimators, i.e. propositions \ref{prop:ConcentrationLSEDiscounted} and \ref{prop:ConcentrationLSEEpisodic}.
The least square error term depends on the optimal value function of $\mdpt$. To get rid of this dependency we decide to control uniformly the error for any optimal value function, which is possible thanks to a net argument.
We introduce a first general lemma that establish a self-normalized martingale uniformly on a set of parameters via a net argument.

\begin{lemma}
\label{lemma:SelfNormalizedProcess}
Consider a sequence of feature vectors $(\phi_t)_{t\ge1}$ in $\R^d$, a set of parameters $\cV\subset\R^\cS$ and for each $V\in\cV$ consider a martingale $(x_t(V))_{t\ge1}$. Finally consider a sequence of positive scalars $(\lambda_t)_{t\ge1}$.
Denote $\Phi_t = \begin{pmatrix} \phi_1 & \dots & \phi_t \end{pmatrix}^\top \in\R^{t\times d}$ and $X_t(V) = \begin{pmatrix} x_1(V) & \dots & x_t(V) \end{pmatrix}^\top \in\R^t$.
Under the following assumptions :
\begin{itemize}
    \item [(i)] there exists a constant $L>0$ such that for any $t\ge1$ and $V\in\cV$, $\|x_t(V)\|_\infty \le L$,
    \item [(ii)] for any $V,V'\in\cV$ and $t\ge1$, $\|x_t(V)-x_t(V')\| \le \|V-V'\|_\infty$,
    \item [(iii)] for any $\epsilon>0$, $\cV$ admits a $\epsilon$-net $\cV_\epsilon$ of finite cardinality $\cN_\epsilon$, i.e. for any $V\in\cV$ there exists $V'\in\cV_\epsilon$ such that $\|V-V'\|_\infty \le \epsilon$,
\end{itemize}
we have for any $\delta>0$, $\epsilon\in(0,L)$ and $t\ge1$ that
\begin{align}
    \proba\left[\max_{V\in\cV} \|\Phi_t^\top X_t(V)\|_{(\Phi_t^\top\Phi_t+\lambda_tI_d)^{-1}}^2
        \le 2L^2\log\left(\frac{\cN_\epsilon}{\delta}\right) + L^2\log\det\left(\left(\Phi_t^\top\Phi_t+\lambda_tI_d\right)\left(\lambda_tI_d\right)^{-1}\right) + td\epsilon^2\right] \ge 1-\delta.
\end{align}
\end{lemma}

\begin{remark}
With the same assumptions as above, if we add that for any $\epsilon\in(0,L)$, $\cN_\epsilon \le \big(1+\frac{2L\sqrt{d}}{\epsilon}\big)^d$, then by choosing $\epsilon = \frac{2L}{\sqrt{t}}$ the threshold becomes
\begin{align}
    L^2\left(2\log\left(\frac{1}{\delta}\right) + 2d\log(e^2(1+\sqrt{dt})) + \log\det\left(\left(\Phi_t^\top\Phi_t+\lambda_tI_d\right)\left(\lambda_tI_d\right)^{-1}\right)\right).
\end{align}
\end{remark}

\begin{remark}
Notice that by linearity of the trace, $\textup{tr}\big(\Phi_t^\top\Phi_t+\lambda_tI_d\big) = \sum_{\ell=1}^t \textup{tr}\big(\phi_\ell\phi_\ell^\top\big) + d\lambda_t = \sum_{\ell=1}^t \|\phi_\ell\|^2 + d\lambda_t \le t + d\lambda_t$.
Now, the trace of a matrix is the sum of its eigenvalues and the determinant their product. The maximum possible product positive scalars, given their sum, is attained when they are all equal. Hence
\begin{align*}
    \det\left(\left(\Phi_t^\top\Phi_t+\lambda_tI_d\right)\left(\lambda_tI_d\right)^{-1}\right)
    \le \frac{1}{\lambda_t^d} \left(\frac{t+d\lambda_t}{d}\right)^d
    = \left(1+\frac{t}{d\lambda_t}\right)^d.
\end{align*}
Therefore, choosing the regularization $\lambda_t = \frac{1}{d}$ and upper bounding $1+\sqrt{dt} \le 2\sqrt{dt}$ and $1+t \le 2t$, the threshold in lemma \ref{lemma:SelfNormalizedProcess} can be replaced with an upper bound
\begin{align}
    L^2\left(2\log\left(\frac{1}{\delta}\right) + d\log\left(8e^4dt^2\right)\right).
\end{align}
\end{remark}

\begin{proof}[Proof of Lemma \ref{lemma:SelfNormalizedProcess}]
The process can be easily controlled when focusing on a single $V\in\cV$ due to a self-normalized martingale concentration result.
In order to control uniformly over the whole set of parameters, we approximate it by a finite net, which raises an error term in the threshold and then control each parameters individually and conclude with a union bound.
In the following, $\delta>0$, $\epsilon\in(0,L)$ and $t\ge1$ are fixed.
Define the events
\begin{align*}
    & \cC_1 = \left\{\max_{V\in\cV} \|\Phi_t^\top X_t(V)\|_{(\Phi_t^\top\Phi_t+\lambda_tI_d)^{-1}}^2
        \le 2L^2\log\left(\frac{\cN_\epsilon}{\delta}\right) + L^2\log\det\left(\left(\Phi_t^\top\Phi_t+\lambda_tI_d\right)\left(\lambda_tI_d\right)^{-1}\right) + td\epsilon^2\right\},\\
    & \cC_2 = \left\{\max_{V\in\cV_\epsilon} \|\Phi_t^\top X_t(V)\|_{(\Phi_t^\top\Phi_t+\lambda_tI_d)^{-1}}^2
        \le 2L^2\log\left(\frac{\cN_\epsilon}{\delta}\right) + L^2\log\det\left(\left(\Phi_t^\top\Phi_t+\lambda_tI_d\right)\left(\lambda_tI_d\right)^{-1}\right)\right\},\\
    & \cC_3(V) = \left\{\|\Phi_t^\top X_t(V)\|_{(\Phi_t^\top\Phi_t\lambda_tI_d)^{-1}}^2
        \le 2L^2\log\left(\frac{\cN_\epsilon}{\delta}\right) + L^2\log\det\left(\left(\Phi_t^\top\Phi_t+\lambda_tI_d\right)\left(\lambda_tI_d\right)^{-1}\right)\right\},
\end{align*}
where the last event is defined for any $V\in\cV_\epsilon$.
Recall that our goal is to show that $\cC_1$ holds with probability at least $1-\delta$.

\medskip
\emph{(i) $\forall V\in\cV_\epsilon, \PP[\cC_3(V)] \ge 1-\frac{\delta}{\cN_\epsilon}$.}
This result is a concentration inequality on self-normalized processes. It can be found as lemma 9 in \cite{NIPS2011_e1d5be1c} for example.
To apply it we use the fact that under all the assumptions, for any $V\in\cV_\epsilon$ we have $\|x_t(V)\| \le L+\epsilon \le 2L$.

\medskip
\emph{(ii) $\PP[\cC_2] \ge 1-\delta$.}
We can immediately see that $\cC_2 = \bigcap_{V\in\cV_\epsilon} \cC_3(V)$. Then an union bound gives
\begin{align*}
    \PP[\cC_2]
    \ge 1 - \sum_{V\in\cV_\epsilon} \big(1 - \PP[\cC_3(V)]\big)
    \ge 1 - \sum_{V\in\cV_\epsilon} \frac{\delta}{\cN_\epsilon}
    = 1 - \delta.
\end{align*}

\medskip
\emph{(iii) $\PP[\cC_1] \ge 1-\delta$.}
We want to show that $\cC_2 \subset \cC_1$. Notice that if $V\in\cV$ and $V'\in\cV_\epsilon$ such that $\|V-V'\|_\infty \le \epsilon$, then by using assumption (ii) we have
\begin{align*}
    \left\|\Phi_t^\top(X_t(V)-X_t(V'))\right\|_{(\Phi_t^\top\Phi_t\lambda_tI_d)^{-1}}^2
    &= \left\|(\Phi_t^\top\Phi_t+\lambda_tI_d)^{-1/2}\Phi_t^\top(X_t(V)-X_t(V'))\right\|^2\\
    &= \sum_{i=1}^d \left|\big((\Phi_t^\top\Phi_t+\lambda_tI_d)^{-1/2}\big)_i^\top\Phi_t^\top(X_t(V)-X_t(V'))\right|^2\\
    &\le \sum_{i=1}^d \left(\sum_{\ell=1}^t \left|\big((\Phi_t^\top\Phi_t+\lambda_tI_d)^{-1/2}\big)_i^\top\phi_\ell\right|\right)^2 \|X_t(V)-X_t(V')\|_\infty^2\\
    &\le t\sum_{i=1}^d \sum_{\ell=1}^t\big(\big((\Phi_t^\top\Phi_t+\lambda_tI_d)^{-1/2}\big)_i^\top\phi_\ell\big)^2 \max_{1\le\ell\le t}\|x_t(V)-x_t(V')\|^2\\
    &\le t\sum_{\ell=1}^t \left\|(\Phi_t^\top\Phi_t+\lambda_tI_d)^{-1/2}\phi_\ell\right\|^2 \|V-V'\|_\infty^2\\
    &\le t\epsilon^2 \textup{tr}\left(\Phi_t(\Phi_t^\top\Phi_t+\lambda_tI_d)^{-1}\Phi_t^\top\right)\\
    &= t\epsilon^2 \textup{tr}\left((\Phi_t^\top\Phi_t+\lambda_tI_d)^{-1}(\Phi_t^\top\Phi_t + \lambda_tI_d - \lambda_tI_d)\right)\\
    &= t\epsilon^2 \left(d - \lambda_t\textup{tr}\left((\Phi_t^\top\Phi_t+\lambda_tI_d)^{-1}\right)\right)\\
    &\le td\epsilon^2,
\end{align*}
and we can finally write
\begin{align*}
    \max_{V\in\cV} \|\Phi_t^\top X_t(V)\|_{(\Phi_t^\top\Phi_t+\lambda_tI_d)^{-1}}^2
    \le \max_{V\in\cV_\epsilon} \|\Phi_t^\top X_t(V)\|_{(\Phi_t^\top\Phi_t+\lambda_tI_d)^{-1}}^2 + td\epsilon^2,
\end{align*}
which implies that $\cC_2 \subset \cC_1$ and conclude the proof.
\end{proof}

In order to use this result we will need to have access to such nets. First we introduce a parametrization of the value function in the linear model.

\begin{lemma}
\label{lemma:Xi}
Under the linear assumptions, any value function can be parametrized by a $d$-dimensional vector the same way the expected rewards and transition probabilities can. Specifically :
\begin{itemize}
    \item In the discounted setup, for any policy $\pi$, there exists a vector $\xi_\mdp^\pi \in \R^d$ such that for any pair $(s,a) \in \cS\times\cA$, $Q_\mdp^\pi(s,a) = \phi(s,a)^\top\xi_\mdp^\pi$.
    Moreover, we have $\xi_\mdp^\pi = \theta_\mdp + \gamma\mu_\mdp^\top V_\mdp^\pi$ and $\|\xi_\mdp^\pi\| \le \frac{\sqrt{d}}{1-\gamma}$.
    \item In the episodic setup, for any policy $\pi$ and any $h\in[H]$, there exists a vector $\xi_{\mdp,h}^\pi \in \R^d$ such that for any pair $(s,a) \in \cS\times\cA$, $Q_{\mdp,h}^\pi(s,a) = \phi(s,a)^\top\xi_{\mdp,h}^\pi$.
    Moreover, we have $\xi_{\mdp,h}^\pi = \theta_{\mdp,h} + \mu_{\mdp,h}^\top V_{\mdp,h+1}^\pi$ and $\|\xi_\mdp^\pi\| \le H\sqrt{d}$.
\end{itemize}
\end{lemma}

\begin{proof}[Proof of lemma \ref{lemma:Xi}]

\medskip

\emph{Discounted Setup}
Using the Bellman equation together with the linear assumptions we directly have $Q_\mdp^\pi(s,a) = \phi(s,a)^\top\big(\theta_\mdp + \gamma\mu_\mdp^\top V_\mdp^\pi\big)$.
Then
\begin{align*}
    \big\|\theta_\mdp + \gamma\mu_\mdp^\top V_\mdp^\pi\big\|
    \le \|\theta_\mdp\| + \gamma\big\|\sum_{s\in\cS}|\mu_\mdp(s)|\big\| \|V_\mdp^\pi\|_\infty
    \le \sqrt{d} + \gamma\frac{\sqrt{d}}{1-\gamma}
    = \frac{\sqrt{d}}{1-\gamma}.
\end{align*}

\medskip
\emph{Episodic Setup}
Using the Bellman equation together with the linear assumptions we directly have $Q_{\mdp,h}^\pi(s,a) = \phi(s,a)^\top\big(\theta_{\mdp,h} + \mu_{\mdp,h}^\top V_{\mdp,h+1}^\pi\big)$.
Then
\begin{align*}
    \big\|\theta_{\mdp,h} + \mu_{\mdp,h}^\top V_{\mdp,h+1}^\pi\big\|
    \le \|\theta_{\mdp,h}\| + \big\|\sum_{s\in\cS}|\mu_{\mdp,h}(s)|\big\| \|V_{\mdp,h+1}^\pi\|_\infty
    \le \sqrt{d} + \sqrt{d}(H-h)
    \le H\sqrt{d}.
\end{align*}
\end{proof}

The existence of the nets is then ensured by the following lemma.

\begin{lemma}
\label{lemma:Net}
Let $\epsilon>0$. The set of optimal value functions admits a finite $\epsilon$-net. Specifically :
\begin{itemize}
    \item[(i)] In the discounted setup, the set $\cV$ of all optimal value functions admits an $\epsilon$-net $\cV_\epsilon$ of cardinality $\cN_\epsilon \le \big(1+\frac{2\sqrt{d}}{(1-\gamma)\epsilon}\big)^d$.
    \item[(ii)] In the episodic setup, the set $\cV(h)$ of all optimal value functions at step $h\in[H]$ admits an $\epsilon$-net $\cV_\epsilon$ of cardinality $\cN_\epsilon \le \big(1+\frac{2H\sqrt{d}}{\epsilon}\big)^d$.
\end{itemize}
\end{lemma}

\begin{proof}[Proof of lemma \ref{lemma:Net}]
\emph{Discounted setup.}
Using the parametrization given by lemma \ref{lemma:Xi} and adding the fact that optimal value functions $(V,Q)$ verify $V(\cdot) = \max_{a\in\cA} Q(\cdot,a)$, we can write
\begin{align*}
    \cV \subset \left\{V\in\R^\cS : \exists\xi\in\R^d, V(\cdot) = \max_{a\in\cA} \phi(\cdot,a)^\top\xi, \|\xi\|\le\frac{\sqrt{d}}{1-\gamma}\right\}.
\end{align*}
If $V,V'\in\cV$ then, considering the corresponding $\xi,\xi'$ as above,
\begin{align*}
    \|V-V'\|_\infty \le \max_{s,a} \|\phi(s,a)^\top(\xi-\xi')\| \le \|\xi-\xi'\|.
\end{align*}
Therefore, using this parametrization by $\xi$, a $\epsilon$-net of $\cV$ can be obtained through a $\epsilon$-net of the euclidean ball of radius $\frac{\sqrt{d}}{1-\gamma}$ in $\R^d$. Such net exists with cardinality $\big(1+\frac{2\sqrt{d}}{(1-\gamma)\epsilon}\big)^d$.

\medskip
\emph{Episodic setup.}
Let $h\in[H]$. With the same reasoning we have
\begin{align*}
    \cV(h) \subset \left\{V\in\R^\cS : \exists\xi\in\R^d, V(\cdot) = \max_{a\in\cA} \phi(\cdot,a)^\top\xi, \|\xi\|\le H\sqrt{d}\right\}
\end{align*}
and a net exists with cardinality $\big(1+\frac{2H\sqrt{d}}{\epsilon}\big)^d$.
\end{proof}

\subsection{Proof of Proposition \ref{prop:ConcentrationLSEDiscounted}}

In this section we show that under any sampling rule the $\frac{1}{d}$-regularized least square estimators verify the following concentration inequality :
For any $\delta\in(0,1)$ the events
\begin{align*}
    \cC(t) = \left\{\distMDPt{\Vt^\star}_{t\Lambda(\omega_t)}^2
    \le \frac{2}{(1-\gamma)^2} \left(2\log\left(\frac{\sqrt{e}\zeta(2)t^2}{\delta}\right) + d\log\left(8e^4dt^2\right)\right)\right\}
\end{align*}
for all $t\ge1$ hold simultaneously with probability at least $1-\delta$.
More precisely we are going to show that for any $t\ge1$, $\PP[\cC(t)] \ge 1-\frac{\delta}{\zeta(2)t^2}$.
The desired result is then shown via a simple union bound over $t$.
It is hard to control this quantity with a dynamic value function, therefore we will control it for all optimal value functions by controlling $\max_{V\in\cV} \big\|\diffMDPt{V}\big\|_{t\Lambda(\omega_t)}^2$ instead and use a net argument. 

\medskip
Denote $\delta_t = \frac{\delta}{\zeta(2)t^2}$ for clarity. Recall the definitions of the $\frac{1}{d}$-regularized least square estimators $\thetat$ and $\mut$ :
\begin{align*}
    &\thetat = \left(\Phi_t^\top\Phi_t + \frac{1}{d}I_d\right)^{-1} \Phi_t^\top R_t,
    &\mut(s) = \left(\Phi_t^\top\Phi_t + \frac{1}{d}I_d\right)^{-1} \Phi_t^\top S_t(s),
\end{align*}
where $\Phi_t = \begin{pmatrix}\phi(s_1,a_1) & \cdots & \phi(s_t,a_t) \end{pmatrix}^\top$, $R_t = \begin{pmatrix} r_1 & \cdots & r_t \end{pmatrix}^\top$ and $S_t(s) = \begin{pmatrix} \delta_{s,s'_1} & \cdots & \delta_{s,s'_t} \end{pmatrix}^\top$. Recall that $t\Lambda(\omega_t) = \Phi_t^\top\Phi_t$.
For any $V\in\cV$
\begin{align*}
    &\diffMDPt{V}\\
    &= \left(\Phi_t^\top\Phi_t + \frac{1}{d}I_d\right)^{-1} \left(\Phi_t^\top R_t - \left(\Phi_t^\top\Phi_t + \frac{1}{d}I_d\right)\theta_\mdp + \gamma\left(\Phi_t^\top S_t^\top - \left(\Phi_t^\top\Phi_t + \frac{1}{d}I_d\right)\mu_\mdp^\top\right)V\right)\\
    &= \left(\Phi_t^\top\Phi_t + \frac{1}{d}I_d\right)^{-1} \Phi_t^\top\left(R_t-\Phi_t\theta_\mdp + \gamma(S_t^\top-\Phi_t\mu_\mdp^\top)V\right) - \frac{1}{d}\left(\Phi_t^\top\Phi_t + \frac{1}{d}I_d\right)^{-1}\left(\theta_\mdp + \gamma\mu_\mdp^\top V\right)\\
    &= \left(\Phi_t^\top\Phi_t + \frac{1}{d}I_d\right)^{-1} \Phi_t^\top X_t(V) - \frac{1}{d}\left(\Phi_t^\top\Phi_t + \frac{1}{d}I_d\right)^{-1} \xi(V)
\end{align*}
where we denote $\xi(V) = \big(\theta_\mdp + \gamma\mu_\mdp^\top V\big)$ and define $x_t(V) = r_t-\phi_t^\top\theta_\mdp + \gamma(V(s'_t)-\phi_t^\top\mu_\mdp^\top V) = r_t-\E[r_t|\cF_{t-1}] + \gamma(V(s'_t)-\E[V(s'_t)|\cF_{t-1}])$ and $X_t(V) = R_t-\Phi_t\theta_\mdp + \gamma(S_t^\top-\Phi_t\mu_\mdp^\top)V = \begin{pmatrix} x_1(V)& \cdots & x_t(V) \end{pmatrix}^\top$.
It follows that
\begin{align*}
    \distMDPt{V}_{\Phi_t^\top\Phi_t}^2
    &\le \left\|\left(\Phi_t^\top\Phi_t + \frac{1}{d}I_d\right)^{-1} \Phi_t^\top X_t(V) - \frac{1}{d}\left(\Phi_t^\top\Phi_t + \frac{1}{d}I_d\right)^{-1}\xi(V)\right\|_{\Phi_t^\top\Phi_t + \frac{1}{d}I_d}^2\\
    &= \left\|\Phi_t^\top X_t(V) - \frac{1}{d}\xi(V)\right\|_{(\Phi_t^\top\Phi_t + \frac{1}{d}I_d)^{-1}}^2\\
    &\le 2\left\|\Phi_t^\top X_t(V)\right\|_{(\Phi_t^\top\Phi_t + \frac{1}{d}I_d)^{-1}}^2 + \frac{2}{d^2}\left\|\xi(V)\right\|_{(\Phi_t^\top\Phi_t + \frac{1}{d}I_d)^{-1}}^2
\end{align*}
Lemma \ref{lemma:Xi} states that $\|\xi(V)\| \le \frac{\sqrt{d}}{1-\gamma}$. Since the greatest eigenvalue of $(\Phi_t^\top\Phi_t + \frac{1}{d}I_d)^{-1}$ can be upper bounded by $d$, we can finally write
\begin{align*}
    \max_{V\in\cV} \distMDPt{V}_{\Phi_t^\top\Phi_t}^2
    \le 2\max_{V\in\cV} \left\|\Phi_t^\top X_t(V)\right\|_{(\Phi_t^\top\Phi_t + \frac{1}{d}I_d)^{-1}}^2 + \frac{2}{(1-\gamma)^2}.
\end{align*}

It is immediate to see that the first two conditions in lemma \ref{lemma:SelfNormalizedProcess} are satisfied by taking $L=(1-\gamma)^{-1}$ and the third one is given by lemma \ref{lemma:Net}.
Therefore we can apply the lemma with $\lambda_t = \frac{1}{d}$ and obtain for all $t\ge1$
\begin{align*}
    \PP\left[\max_{V\in\cV} \left\|\Phi_t^\top X_t(V)\right\|_{(\Phi_t^\top\Phi_t + \frac{1}{d}I_d)^{-1}}^2 \le \frac{1}{(1-\gamma)^2}\left(2\log\left(\frac{1}{\delta_t}\right) + d\log\left(8e^4dt^2\right)\right)\right] \ge 1 - \delta_t.
\end{align*}

Using the bound above, this event directly implies $\cC(t)$ and we can finally conclude that $\PP[\cC(t)] \ge 1-\delta_t$ for all $t\ge1$.

\subsection{Proof of Proposition \ref{prop:ConcentrationLSEEpisodic}}
In this section we show that under any sampling rule the $\frac{1}{d}$-regularized least square estimators verify the following concentration inequality :
For any $\delta\in(0,1)$ and any $h\in[H]$ the events
\begin{align*}
    \cC(t) = \left\{\distMDPtH{\VtH{h+1}^\star}_{t\Lambda(\omega_t)}^2
    \le 2H^2 \left(2\log\left(\frac{\sqrt{e}\zeta(2)t^2}{\delta}\right) + d\log\left(8e^4dt^2\right)\right)\right\}
\end{align*}
for all $t\ge1$ hold simultaneously with probability at least $1-\delta$.
As in the discounted setup it is enough to show that for any $t\ge1$, $\PP[\cC(t)] \ge 1-\frac{\delta}{\zeta(2)t^2}$ and we will follow the same reasoning.

\medskip
Denote $\delta_t = \frac{\delta}{\zeta(2)t^2}$ for clarity. Recall the definitions of the $\frac{1}{d}$-regularized least square estimators $\thetat$ and $\mut$ :
\begin{align*}
    &\thetath = \left(\Phi_t^\top\Phi_t + \frac{1}{d}I_d\right)^{-1} \Phi_t^\top R_{t,h},
    &\muth(s) = \left(\Phi_t^\top\Phi_t + \frac{1}{d}I_d\right)^{-1} \Phi_t^\top S_{t,h}(s),
\end{align*}
where $\Phi_t = \begin{pmatrix}\phi(s_1,a_1) & \cdots & \phi(s_t,a_t) \end{pmatrix}^\top$, $R_{t,h} = \begin{pmatrix} r_{1,h} & \cdots & r_{t,h} \end{pmatrix}^\top$ and $S_{t,h}(s) = \begin{pmatrix} \delta_{s,s'_{1,h}} & \cdots & \delta_{s,s'_{t,h}} \end{pmatrix}^\top$. Recall that $t\Lambda(\omega_t) = \Phi_t^\top\Phi_t$.
For any $V\in\cV(h+1)$
\begin{align*}
    &\diffMDPtH{V}\\
    &= \left(\Phi_t^\top\Phi_t + \frac{1}{d}I_d\right)^{-1} \left(\Phi_t^\top R_t - \left(\Phi_t^\top\Phi_t + \frac{1}{d}I_d\right)\theta_\mdp + \left(\Phi_t^\top S_t^\top - \left(\Phi_t^\top\Phi_t + \frac{1}{d}I_d\right)\mu_\mdp^\top\right)V\right)\\
    &= \left(\Phi_t^\top\Phi_t + \frac{1}{d}I_d\right)^{-1} \Phi_t^\top\left(R_t-\Phi_t\theta_\mdp + (S_t^\top-\Phi_t\mu_\mdp^\top)V\right) - \frac{1}{d}\left(\Phi_t^\top\Phi_t + \frac{1}{d}I_d\right)^{-1}\left(\theta_\mdp + \mu_\mdp^\top V\right)\\
    &= \left(\Phi_t^\top\Phi_t + \frac{1}{d}I_d\right)^{-1} \Phi_t^\top X_t(V) - \frac{1}{d}\left(\Phi_t^\top\Phi_t + \frac{1}{d}I_d\right)^{-1} \xi(V)
\end{align*}
where we denote $\xi(V) = \big(\theta_\mdp + \mu_\mdp^\top V\big)$ and define $x_t(V) = r_t-\phi_t^\top\theta_\mdp + (V(s'_t)-\phi_t^\top\mu_\mdp^\top V) = r_t-\E[r_t|\cF_{t-1}] + (V(s'_t)-\E[V(s'_t)|\cF_{t-1}])$ and $X_t(V) = R_t-\Phi_t\theta_\mdp + (S_t^\top-\Phi_t\mu_\mdp^\top)V = \begin{pmatrix} x_1(V)& \cdots & x_t(V) \end{pmatrix}^\top$.
It follows that
\begin{align*}
    \distMDPtH{V}_{\Phi_t^\top\Phi_t}^2
    &\le \left\|\left(\Phi_t^\top\Phi_t + \frac{1}{d}I_d\right)^{-1} \Phi_t^\top X_t(V) - \frac{1}{d}\left(\Phi_t^\top\Phi_t + \frac{1}{d}I_d\right)^{-1}\xi(V)\right\|_{\Phi_t^\top\Phi_t + \frac{1}{d}I_d}^2\\
    &= \left\|\Phi_t^\top X_t(V) - \frac{1}{d}\xi(V)\right\|_{(\Phi_t^\top\Phi_t + \frac{1}{d}I_d)^{-1}}^2\\
    &\le 2\left\|\Phi_t^\top X_t(V)\right\|_{(\Phi_t^\top\Phi_t + \frac{1}{d}I_d)^{-1}}^2 + \frac{2}{d^2}\left\|\xi(V)\right\|_{(\Phi_t^\top\Phi_t + \frac{1}{d}I_d)^{-1}}^2
\end{align*}
Lemma \ref{lemma:Xi} states that $\|\xi(V)\| \le H\sqrt{d}$. Since the greatest eigenvalue of $(\Phi_t^\top\Phi_t + \frac{1}{d}I_d)^{-1}$ can be upper bounded by $d$, we can finally write
\begin{align*}
    \max_{V\in\cV} \distMDPtH{V}_{\Phi_t^\top\Phi_t}^2
    \le 2\max_{V\in\cV} \left\|\Phi_t^\top X_t(V)\right\|_{(\Phi_t^\top\Phi_t + \frac{1}{d}I_d)^{-1}}^2 + 2H^2.
\end{align*}

It is immediate to see that the first two conditions in lemma \ref{lemma:SelfNormalizedProcess} are satisfied by taking $L=H$ and the third one is given by lemma \ref{lemma:Net}.
Therefore we can apply the lemma with $\lambda_t = \frac{1}{d}$ and obtain for all $t\ge1$
\begin{align*}
    \PP\left[\max_{V\in\cV} \left\|\Phi_t^\top X_t(V)\right\|_{(\Phi_t^\top\Phi_t + \frac{1}{d}I_d)^{-1}}^2 \le H^2\left(2\log\left(\frac{1}{\delta_t}\right) + d\log\left(8e^4dt^2\right)\right)\right] \ge 1 - \delta_t.
\end{align*}

Using the bound above, this event directly implies $\cC(t)$ and we can finally conclude that $\PP[\cC(t)] \ge 1-\delta_t$ for all $t\ge1$.

\subsection{Concentration bounds for self-normalized processes}

 Proposition \ref{prop:sn} is borrowed from  \cite{abbasi2011improved}.

\begin{prop}\label{prop:sn}
Let $( \cF_t )_{t\ge 1}$ be a filtration. Let $ (\eta_t  )_{t\ge 1}$ be a stochastic process adapted to $ (\cF_t)_{t\ge 1}$ and taking values in $\RR$. Let $( \phi_t )_{t \ge 1}$ be a predictable stochastic process with respect to $( \cF_t )_{t\ge 1}$, taking values in $\RR^d$. Furthermore, assume that $\eta_{t+1}$, conditionally on $\cF_t$, is a zero-mean, $\sigma^2$-sub-gaussian \footnote{We say that a random variable is $\sigma^2$-sub-gaussian, if for all $\lambda \in \RR$, $\EE[\exp(\lambda X )] \le \exp\left( \lambda^2 \sigma^2 / 2 \right)$.}. Then, for all $\delta \in (0,1)$,
\begin{align*}
    \PP \left( \left\Vert \left(\sum_{\ell=1}^t\phi_\ell \phi_\ell^\top + \lambda I_d \right)^{-1/2}   \left(\sum_{\ell = 1}^t \phi_\ell \eta_\ell  \right)\right\Vert^2        \le 2 \sigma^2 \log\left( \frac{\det( (\lambda^{-1} (\sum_{\ell = 1}^t \phi_\ell \phi_\ell^\top ) + I_d) )}{\delta}\right) \right) \ge 1 - \delta.
\end{align*}
\end{prop}

Proposition \label{prop:sn+} is a direct consequence of Proposition \ref{prop:sn} and may be obtained via a net argument.

\begin{prop}\label{prop:sn+}
Under the same assumptions as in Proposition \ref{prop:sn}, with the only exception that $(\eta_t)_{t\ge 1}$ is now taking values in $\RR^p$, and for each $t \ge 1$, the random vector $\eta_{t+1}$, conditionally on $\cF_t$, is a zero-mean, and $\sigma^2$-subgaussian \footnote{We say that a random vector $X$ taking values in $\RR^d$, is  $\sigma^2$-subgaussian if for all $\theta \in \RR^{d}$,  $\EE[\exp( \theta^\top X )] \le \exp(\Vert \theta \Vert^2 \sigma/2)$}. Then,  
\begin{align*}
    \PP \left( \left\Vert \left(\sum_{\ell=1}^t\phi_\ell \phi_\ell^\top + \lambda I_d \right)^{-1/2}   \left(\sum_{\ell = 1}^t \phi_\ell \eta_\ell^\top  \right)\right\Vert^2      \le 4 \sigma^2 \log\left( \frac{5^p\det( (\lambda^{-1} (\sum_{\ell = 1}^t \phi_\ell \phi_\ell^\top ) + I_d) )}{\delta}\right) \right) \ge 1 - \delta.
\end{align*}
\end{prop}

\newpage
\section{Sample Complexity Analysis}

In this section we establish the sample complexity bounds of theorems \ref{thm:SampleComplexityDiscounted}, \ref{thm:SampleComplexityNavigation} and \ref{thm:SampleComplexityEpisodic}.
First, we establish a continuity-like bound on the quantity $U(\mdpt,\omega_t)^{-1}$ which follow the same reasoning as appendix \ref{app:LowerBound}. The first step is the continuity of the gap, given by the following lemma.

\begin{lemma}
\label{lemma:GapDiff}
\begin{itemize}
    \item [(i)] In the discounted setting, it holds that
    \begin{equation}
        |\Delta(\mdpt)-\Delta(\mdp)|
        \le \frac{2}{1-\gamma}\max_{s,a}\left|\phi(s,a)\left(\diffMDPt{\Vt^\star}\right)\right|.
    \end{equation}
    \item [(ii)] In the episodic setting, there exists $h \in [H]$ such that the following holds
    \begin{equation}
        |\Delta(\mdpt)-\Delta(\mdp)|
        \le 2\sum_{h=1}^H\max_{s,a}\left|\phi(s,a)\left(\diffMDPtH{\VtH{h+1}^\star}\right)\right|.
    \end{equation}
\end{itemize}
\end{lemma}

\begin{proof}[Proof of Lemma \ref{lemma:GapDiff}]
We present the proofs of \emph{(i)} and \emph{(ii)} separately.

\medskip 
\emph{Discounted setting - proof of (i).}
For clarity we denote for both MDPs, for any $(s,a)\in\cS\times\cA$, $\Delta_{s,a} = V^\star(s)-Q^\star(s,a)$, so that $\Delta = \min_{s\in\cS, a\neq\pi^\star(s)} \Delta_{s,a}$.

Let $(s,a)$ be the pair such that $\Delta(\mdp) = \Delta_{s,a}(\mdp)$.
If $a \neq \pi_t^\star(s)$ then $\Delta(\mdpt) \le \Delta_{s,a}(\mdpt)$ and $\Delta(\mdpt)-\Delta(\mdp) \le \Delta_{s,a}(\mdpt)-\Delta_{s,a}(\mdp)$.
Else, since both MDPs have exactly $|\cS|$ optimal state/action pairs (one for each state), the fact that the pair $(s,a)$ is optimal for $\mdpt$ but not for $\mdp$ means that there exists a pair $(s',a')$ optimal for $\mdp$ but not for $\mdpt$, and we have $\Delta(\mdpt)-\Delta(\mdp) \le \Delta_{s',a'}(\mdpt) = \Delta_{s',a'}(\mdpt)-\Delta_{s',a'}(\mdp)$.
Either way, and doing the same reasoning to bound $\Delta(\mdp)-\Delta(\mdpt)$, we can find a pair $(s,a)$ such that
\begin{align*}
    |\Delta(\mdpt)-\Delta(\mdp)|
    &\le |\Delta_{s,a}(\mdpt)-\Delta_{s,a}(\mdp)|\\
    &= |\Vt^\star(s)-\Qt^\star(s,a) - V_\mdp^\star(s)+Q_\mdp^\star(s,a)|\\
    &= |\Vt^\star(s)-V_\mdp^\star(s) + Q_\mdp^\star(s,a)-\Qt^\star(s,a)|\\
    &\le \|\Vt^\star-V_\mdp^\star\|_\infty + \|\Qt^\star-Q_\mdp^\star\|_\infty.
\end{align*}
We conclude with lemma \ref{lemma:ValueDiffOptimalPolicy}.

\medskip 
\emph{Episodic setting - proof of (ii).}
For clarity we denote for both MDPs, for any $h\in[H]$ and $(s,a)\in\cS\times\cA$, $\Delta_{h,s,a} = V_h^\star(s)-Q_h^\star(s,a)$, so that $\Delta = \min_{h\in[H], s\in\cS, a\neq\pi^\star(s)} \Delta_{h,s,a}$.

Let $(h,s,a)$ be the parameters such that $\Delta(\mdp) = \Delta_{h,s,a}(\mdp)$.
If $a \neq \pi_{t,h}^\star(s)$ then $\Delta(\mdpt) \le \Delta_{h,s,a}(\mdpt)$ and $\Delta(\mdpt)-\Delta(\mdp) \le \Delta_{h,s,a}(\mdpt)-\Delta_{h,s,a}(\mdp)$.
Else, since both MDPs have exactly $H|\cS|$ optimal state/action pairs (one for each (step, state) pair), the fact that $(h,s,a)$ is optimal for $\mdpt$ but not for $\mdp$ means that there exist $(h',s',a')$ optimal for $\mdp$ but not for $\mdpt$, and we have $\Delta(\mdpt)-\Delta(\mdp) \le \Delta_{h',s',a'}(\mdpt) = \Delta_{h',s',a'}(\mdpt)-\Delta_{s',a'}(\mdp)$.
Either way, and doing the same reasoning to bound $\Delta(\mdp)-\Delta(\mdpt)$, we can find $(h,s,a)$ such that
\begin{align*}
    |\Delta(\mdpt)-\Delta(\mdp)|
    &\le |\Delta_{h,s,a}(\mdpt)-\Delta_{h,s,a}(\mdp)|\\
    &= |\VtH{h}^\star(s)-\QtH{h}^\star(s,a) - V_{\mdp,h}^\star(s)+Q_{\mdp,h}^\star(s,a)|\\
    &= |\VtH{h}^\star(s)-V_{\mdp,h}^\star(s) + Q_{\mdp,h}^\star(s,a)-\QtH{h}^\star(s,a)|\\
    &\le \|\Vt^\star-V_{\mdp,h}^\star\|_\infty + \|\QtH{h}^\star-Q_{\mdp,h}^\star\|_\infty.
\end{align*}
We conclude with lemma \ref{lemma:ValueDiffOptimalPolicy}.
\end{proof}

\begin{lemma}
\label{lemma:UDiff}
For any $t \ge 1$ it holds that 
\begin{equation}
    \left|U^\star(\mdp)^{-1} - U\big(\mdpt,\omega_t\big)^{-1}\right|
    \le B(t),
\end{equation}
where $B(t)$ is defined: 
\begin{itemize}
    \item [(i)] in the discounted setting as 
    \begin{equation}
        B(t)
        = 6(1-\gamma)^2\distMDPt{\Vt^\star}_{\Lambda(\omega_t)}^2 + \left(\frac{5}{4}-\frac{\sigma(\omega^\star)}{\sigma(\omega_t)}\right) U^\star(\mdp)^{-1}.
        \end{equation}
    \item [(ii)] in the episodic setting as
    \begin{equation}
        B(t)
        = \frac{6}{H^2}\sum_{h=1}^H\distMDPtH{\VtH{h+1}^\star}_{\Lambda(\omega_t)}^2 + \left(\frac{5}{4}-\frac{\sigma(\omega^\star)}{\sigma(\omega_t)}\right) U^\star(\mdp)^{-1}.
    \end{equation}
\end{itemize}
\end{lemma}

\begin{proof}
For both settings we can write
\begin{align*}
    \left|U^\star(\mdp)^{-1} - U\big(\mdpt,\omega_t\big)^{-1}\right|
    &\le \left|U^\star(\mdp)^{-1} - U(\mdp,\omega_t)^{-1}\right| + \left|U(\mdp,\omega_t)^{-1} - U\big(\mdpt,\omega_t\big)^{-1}\right|
\end{align*}
and the first term can be rewritten
\begin{align*}
    \left|U^\star(\mdp)^{-1} - U\left(\mdp,\omega_t\right)^{-1}\right|
    = \left(1 - U^\star(\mdp)U(\mdp,\omega_t)^{-1}\right) U^\star(\mdp)^{-1}
    = \left(1-\frac{\sigma(\omega^\star)}{\sigma(\omega_t)}\right) U^\star(\mdp)^{-1}.
\end{align*}

For the second term, regardless of the setting there exists a quantity $u(\omega_t)^{-1}$ such that $U(\cdot,\omega_t)^{-1} = u(\omega_t)^{-1}(\Delta(\cdot)+\varepsilon)^2$. In the discounted setting we have $u(\omega_t)^{-1} = \frac{3(1-\gamma)^4}{10\sigma(\omega_t)}$ and in the episodic setting we have $u(\omega_t)^{-1} = \frac{3}{10H^3\sigma(\omega_t)}$.
Then
\begin{align*}
    &\left|U(\mdp,\omega_t)^{-1} - U\big(\mdpt,\omega_t\big)^{-1}\right|\\
    &= u(\omega_t)^{-1} \left|(\Delta(\mdp)+\varepsilon)^2-(\Delta(\mdpt)+\varepsilon)^2\right|\\
    &= u(\omega_t)^{-1} \left|\left(\Delta(\mdpt)+\Delta(\mdp)+2\varepsilon\right) \left(\Delta(\mdpt)-\Delta(\mdp)\right)\right|\\
    &= u(\omega_t)^{-1} \left|\big(\Delta(\mdpt)-\Delta(\mdp)\big)^2 + 2(\Delta(\mdp)+\varepsilon)\big(\Delta(\mdpt)-\Delta(\mdp)\big)\right|\\
    &\le u(\omega_t)^{-1} \big(\Delta(\mdpt)-\Delta(\mdp)\big)^2
        + 2\sqrt{\frac{1}{4}u(\omega_t)^{-1} (\Delta(\mdp)+\varepsilon)^2} \sqrt{4u(\omega_t)^{-1} \big(\Delta(\mdpt)-\Delta(\mdp)\big)^2}\\
    &\le u(\omega_t)^{-1} \big(\Delta(\mdpt)-\Delta(\mdp)\big)^2
        + \frac{1}{4}U(\mdp,\omega_t)^{-1} + 4u(\omega_t)^{-1} \big(\Delta(\mdpt)-\Delta(\mdp)\big)^2\\
    &\le 5u(\omega_t)^{-1} \big(\Delta(\mdpt)-\Delta(\mdp)\big)^2 + \frac{1}{4}U^\star(\mdp)^{-1}
\end{align*}
using $U(\mdp,\omega_t)^{-1} \le U^\star(\mdp)^{-1}$ for the last step.
Now we separate the end of the proof for both settings.

\medskip 
\emph{Discounted setting.}
To conclude we need to show that
\begin{align*}
    \frac{3(1-\gamma)^4}{2\sigma(\omega_t)} \big(\Delta(\mdpt)-\Delta(\mdp)\big)^2
    \le 6\distMDPt{\Vt^\star}_{\Lambda(\omega_t)}^2.
\end{align*}
Recall that lemma \ref{lemma:GapDiff} gives
\begin{align*}
    \big|\Delta(\mdpt)-\Delta(\mdp)\big| \le \frac{2}{1-\gamma} \max_{s,a} \left|\phi(s,a)^\top\big(\diffMDPt{\Vt^\star}\big)\right|.
\end{align*}
Then the result is obtained by applying lemma \ref{lemma:Optimization} with $n=1$, $\phi_1$ the feature maximizing the term above and $\Delta = \frac{1-\gamma}{2}\big|\Delta(\mdpt)-\Delta(\mdp)\big|$.

\medskip 
\emph{Episodic setting.}
To conclude we need to show that
\begin{align*}
    \frac{3}{2H^3\sigma(\omega_t)} \big(\Delta(\mdpt)-\Delta(\mdp)\big)^2
    \le 6\sum_{h=1}^H\distMDPtH{\VtH{h+1}^\star}_{\Lambda(\omega_t)}^2.
\end{align*}
Recall that lemma \ref{lemma:GapDiff} gives
\begin{align*}
    \big|\Delta(\mdpt)-\Delta(\mdp)\big| \le 2\sum_{h=1}^H \max_{s,a} \left|\phi(s,a)^\top\big(\diffMDPtH{\VtH{h+1}^\star}\big)\right|.
\end{align*}
Then the result is obtained by applying lemma \ref{lemma:Optimization} with $n=H$, $\phi_h$ the feature maximizing the $h$-term in the sum above and $\Delta = \frac{1}{2}\big|\Delta(\mdpt)-\Delta(\mdp)\big|$.
\end{proof}

\subsection{Proof of Theorem \ref{thm:SampleComplexityDiscounted}}

Recall the threshold
\begin{align*}
    \beta(\delta,t) = \frac{12}{5} \left(2\log\left(\frac{\sqrt{e}\zeta(2)t^2}{\delta}\right) + d\log\left(8e^4dt^2\right)\right)
\end{align*}
and the stopping time
\begin{align*}
    \tau = \inf\left\{t \ge 1 : Z(t) > \beta(\delta,t)\right\},
\end{align*}
where $Z(t) = t\,U(\mdpt,\omega_t)^{-1}$.
In order to establish the sample complexity upper bound, we are going to find a time $T$ such that for any $t\ge T$, $\proba[\tau>t] = O\big(\frac{1}{t^2}\big)$, so that we can bound $\E[\tau]$ by $T$ plus a constant.
Thanks to lemma \ref{lemma:UDiff}, we have $\{\tau>t\} \subset \big\{t\,U(\mdpt,\omega_t)^{-1} \le \beta(\delta,t)\big\} \subset \left\{t\,U^\star(\mdp)^{-1} \le \beta(\delta,t) + tB(t)\right\}$.
Recall that when proving proposition \ref{prop:ConcentrationLSEDiscounted} we have shown that for any $\delta'>0$ and for any $t\ge1$ we have with probability at least $1-\frac{\delta'}{\zeta(2)t^2}$ that
\begin{align*}
    \distMDPt{\Vt^\star}_{t\Lambda(\omega_t)}^2
    \le \frac{5}{6(1-\gamma)^2} \beta(\delta',t).
\end{align*}
Moreover, lemma \ref{lemma:ConcentrationLambdaGeneral} states that if $t \ge \frac{28d}{3} \log\big(\frac{2\zeta(2)dt^2}{\delta'}\big)$ then with probability at least $1-\frac{\delta'}{\zeta(2)t^2}$ we have $c(\omega_t) \le 2c(\omega^\star)$.
Choosing $\delta'=1$ and plugging both bounds in the definition of $B(t)$ we have with an union bound that for all $t\ge T_1$
\begin{align*}
    \PP\left[tB(t) \le 5\beta(1,t) + \frac{3t}{4}U^\star(\mdp)^{-1}\right] \ge 1 - \frac{2}{\zeta(2)t^2},
\end{align*}
where we define $T_1 = \frac{56d}{3} \log(2\zeta(2)d) + \frac{112d}{3} \log\big(\frac{112d}{3}\big) = \frac{56d}{3} \log\big(\frac{6272\zeta(2)d^3}{3}\big)$, so that according to lemma \ref{lemma:LogarithmBound} $t\ge T_1$ implies $t \ge \frac{28d}{3} \log(2\zeta(2)dt^2)$.
Now to conclude we only need to show that this event implies $\left\{t\,U^\star(\mdp)^{-1} > \beta(\delta,t) + tB(t)\right\}$ when $t$ is large enough.
Assume that $tB(t) \le 5\beta(1,t) + \frac{3t}{4}U^\star(\mdp)^{-1}$. Since $\delta<1$ we have $\beta(1,t) < \beta(\delta,t)$ and $\beta(\delta,t) + tB(t) \le 6\beta(\delta,t) + \frac{3t}{4}U^\star(\mdp)^{-1}$.
To show that this is bounded by $t\,U^\star(\mdp)^{-1}$ is equivalent to show that $24\beta(\delta,t) \le t\,U^\star(\mdp)^{-1}$. Again, we can show that this last bound is true when $t\ge T_2$ thanks to lemma \ref{lemma:LogarithmBound}, where we define
\begin{align*}
    T_2
    = U^\star(\mdp)\frac{576}{5} \left(2\log\left(\frac{\sqrt{e}\zeta(2)}{\delta}\right) + d\log(8e^4d)\right) + U^\star(\mdp)\frac{576(d+2)}{5} \log\left(\frac{576(d+2)}{5}\right).
\end{align*}

We have shown that when $t\ge\max(T_1,T_2)$
\begin{align*}
    \PP[\tau>t]
    \le \PP\left[t\,U^\star(\mdp)^{-1} \le \beta(\delta,t) + tB(t)\right]
    \le \PP\left[tB(t) > 5\beta(1,t) + \frac{3t}{4}U(\mdpt,\omega_t)^{-1}\right]
    \le \frac{2}{\zeta(2)t^2}.
\end{align*}
Therefore, denoting $T = \max(T_1,T_2)$,
\begin{align*}
    \E[\tau]
    = \sum_{t\ge0} \proba[\tau>t]
    = \sum_{t=0}^{T-1} \proba[\tau>t] + \sum_{t=T}^{+\infty} \proba[\tau>t]
    \le T + \sum_{t=T}^{+\infty} \frac{2}{\zeta(2)t^2}
    \le T + 2.
\end{align*}

\subsection{Proof of Theorem \ref{thm:SampleComplexityEpisodic}}
\label{app:SampleComplexityEpisodic}

Recall the threshold
\begin{align*}
    \beta(\delta,t) = \frac{12}{5} \left(2\log\left(\frac{\sqrt{e}\zeta(2)t^2}{\delta}\right) + d\log\left(8e^4dt^2\right)\right)
\end{align*}
and the stopping time
\begin{align*}
    \tau = \inf\left\{t \ge 1 : Z(t) > H\beta(\delta/H,t)\right\},
\end{align*}
where $Z(t) = t\,U(\mdpt,\omega_t)^{-1}$.
Thanks to lemma \ref{lemma:UDiff}, we have $\{\tau>t\} \subset \big\{t\,U(\mdpt,\omega_t)^{-1} \le H\beta(\delta/H,t)\big\} \subset \left\{t\,U^\star(\mdp)^{-1} \le H\beta(\delta/H,t) + tB(t)\right\}$.
Recall that from the proof of proposition \ref{prop:ConcentrationLSEEpisodic} we can deduce that for any $\delta'>0$ and for any $t\ge1$ we have with probability at least $1-\frac{\delta'}{\zeta(2)t^2}$ that
\begin{align*}
    \sum_{h=1}^H \distMDPtH{\VtH{h+1}^\star}_{t\Lambda(\omega_t)}^2
    \le \frac{5H^3}{6} \beta(\delta'/H,t).
\end{align*}
Moreover, lemma \ref{lemma:ConcentrationLambdaGeneral} states that if $t \ge \frac{28d}{3} \log\big(\frac{2\zeta(2)dt^2}{\delta'}\big)$ then with probability at least $1-\frac{\delta'}{\zeta(2)t^2}$ we have $c(\omega_t) \le 2c(\omega^\star)$.
Choosing $\delta'=1$ and plugging both bounds in the definition of $B(t)$ we have with an union bound that for all $t\ge T_1$
\begin{align*}
    \PP\left[tB(t) \le 5H\beta(1/H,t) + \frac{3t}{4}U^\star(\mdp)^{-1}\right] \ge 1 - \frac{2}{\zeta(2)t^2},
\end{align*}
where we define $T_1 = \frac{56d}{3} \log(2\zeta(2)d) + \frac{112d}{3} \log\big(\frac{112d}{3}\big) = \frac{56d}{3} \log\big(\frac{6272\zeta(2)d^3}{3}\big)$, so that according to Lemma \ref{lemma:LogarithmBound} $t\ge T_1$ implies $t \ge \frac{28d}{3} \log(2\zeta(2)dt^2)$.
Now to conclude we only need to show that this event implies $\left\{t\,U^\star(\mdp)^{-1} > H\beta(\delta/H,t) + tB(t)\right\}$ when $t$ is large enough.
Assume that $tB(t) \le 5H\beta(1/H,t) + \frac{3t}{4}U^\star(\mdp)^{-1}$. Since $\delta<1$ we have $\beta(1/H,t) < \beta(\delta/H,t)$ and $H\beta(\delta/H,t) + tB(t) \le 6H\beta(\delta/H,t) + \frac{3t}{4}U^\star(\mdp)^{-1}$.
To show that this is bounded by $t\,U^\star(\mdp)^{-1}$ is equivalent to show that $24H\beta(\delta/H,t) \le t\,U^\star(\mdp)^{-1}$. Again, we can show that this last bound is true when $t\ge T_2$ thanks to lemma \ref{lemma:LogarithmBound}, where we define
\begin{align*}
    T_2
    = U^\star(\mdp)\frac{576H}{5} \left(2\log\left(\frac{\sqrt{e}\zeta(2)H}{\delta}\right) + d\log(8e^4d)\right) + U^\star(\mdp)\frac{576(d+2)H}{5} \log\left(\frac{576(d+2)}{5}\right).
\end{align*}

We have shown that when $t\ge\max(T_1,T_2)$
\begin{align*}
    \PP[\tau>t]
    \le \PP\left[t\,U^\star(\mdp)^{-1} \le H\beta(\delta/H,t) + tB(t)\right]
    \le \PP\left[tB(t) > 5H\beta(1/H,t) + \frac{3t}{4}U^\star(\mdp)^{-1}\right]
    \le \frac{2}{\zeta(2)t^2}.
\end{align*}
Therefore, denoting $T = \max(T_1,T_2)$,
\begin{align*}
    \E[\tau]
    = \sum_{t\ge0} \proba[\tau>t]
    = \sum_{t=0}^{T-1} \proba[\tau>t] + \sum_{t=T}^{+\infty} \proba[\tau>t]
    \le T + \sum_{t=T}^{+\infty} \frac{2}{\zeta(2)t^2}
    \le T + 2.
\end{align*}

\subsection{Technical lemmas}

\begin{lemma}
\label{lemma:LogarithmBound}
Let $a,b > 0$.
A sufficient condition for $t > a\log(t)+b$ to hold is that $t \ge 2a\log(2a)+2b$.
\end{lemma}

\begin{proof}
Let $t \ge 2a\log(2a)+2b$. Then
\begin{align*}
    t
    \ge a\frac{t}{2a} + \frac{t}{2}
    > a\log\left(\frac{t}{2a}\right) + a\log(2a) + b \ge a\log(t) + b.
\end{align*}
\end{proof}

\end{document}